\documentclass[graybox]{./styles/svmult}

\usepackage{mathptmx}       
\usepackage{helvet}         
\usepackage{courier}        
\usepackage{type1cm}        
\usepackage{makeidx}         
\usepackage{graphicx}        
\usepackage{multicol}        
\usepackage[bottom]{footmisc}

\usepackage{times}
\usepackage{amsmath}
\usepackage{amssymb}
\usepackage{algorithmic}
\usepackage{algorithm}
\usepackage{subfigure}
\usepackage{multirow}

\makeindex             

\usepackage[pagebackref=true,breaklinks=true,letterpaper=true,colorlinks,bookmarks=false]{hyperref}

\begin{document}

\newcommand{\mytextbf}[1]{\vspace{.1cm}\noindent\textbf{{#1}}~}

\title*{Fast Approximate $K$-Means via Cluster Closures}

\author{
Jingdong Wang
and
Jing Wang
and
Qifa Ke
and
Gang Zeng
and
Shipeng Li}

\institute{Jingdong Wang \at Microsoft \email{jingdw@microsoft.com}
\and
Jing Wang,
\at Peking University
\email{cis.wangjing@pku.edu.cn}
\and
Qifa Ke
\at Microsoft \email{qke@microsoft.com}
\and
Gang Zeng, \at Peking University
\email{g.zeng@ieee.org}
\and
Shipeng Li
\at Microsoft \email{spli@microsoft.com}
}

\maketitle

\abstract{$K$-means, a simple and effective clustering algorithm,
is one of the most widely used algorithms in multimedia and computer vision community.
Traditional $k$-means is an iterative algorithm---in each iteration
new cluster centers are computed
and each data point is re-assigned to its nearest center. The cluster re-assignment step becomes prohibitively expensive when the number of data points and cluster centers are large.\\
\indent In this paper, we propose a novel approximate $k$-means algorithm
to greatly reduce the computational complexity in the assignment step.
Our approach is motivated by the observation that most active points
changing their cluster assignments at each iteration are located on or near cluster boundaries.
The idea is to efficiently identify those active points by pre-assembling the data into groups of neighboring points using multiple random spatial partition trees,
and to use the neighborhood information to construct a closure for each cluster,
in such a way only a small number of cluster candidates need to be considered
when assigning a data point to its nearest cluster. Using complexity analysis,
image data clustering,
and applications to image retrieval,
we show that our approach out-performs state-of-the-art approximate $k$-means algorithms
in terms of clustering quality and efficiency.
}

\section{Introduction}

$K$-means~\cite{MacQueen67} has been widely used
in multimedia,
computer vision and machine learning for clustering and vector quantization.
In large-scale image retrieval,
it is advantageous to learn a large codebook containing one million or more entries~\cite{NisterS06,PhilbinCISZ07,SivicZ03},
which requires clustering tens or
even hundreds of millions of high-dimensional feature descriptors
into one million or more clusters.
Another emerging application of large-scale clustering is
to organize a large corpus of web images for various purposes such as web image browsing/exploring~\cite{WangJH11}.

The standard $k$-means algorithm,
Lloyd's algorithm~\cite{Forgy65, Lloyd82, MacQueen67},
is an iterative refinement approach that greedily
minimizes the sum of squared distances between each point and its assigned cluster center.
It consists of two iterative steps,
the assignment step and the update step.
The assignment step aims to find the nearest cluster
for each point by checking the distance
between the point and each cluster center;
The update step re-computes the cluster centers
based on current assignments.
When clustering $n$ points into $k$ clusters,
the assignment step costs $O(nk)$.
For applications with large $nk$,
the assignment step in exact $k$-means becomes prohibitively expensive.
Therefore many approximate solutions,
such as hierarchial $k$-means (HKM)~\cite{NisterS06}
and approximate $k$-means (AKM)~\cite{PhilbinCISZ07},
have been developed.

In this paper,
we introduce a novel and effective approximate $k$-means algorithm~\footnote{A conference version appeared in~\cite{WangWKZL12}.}.
Our approach is motivated by
the observation
that~\emph{active points},
defined as the points whose cluster assignments change in each iteration,
often locate at or near boundaries of different clusters.
The idea is to identify those
active points at or near cluster boundaries to improve both the efficiency and accuracy
in the assignment step of the $k$-means algorithm.
We generate a neighborhood set for each data point by pre-assembling the data points using multiple random partition trees~\cite{VermaKD09}.
A \emph{cluster closure} is then formed by expanding each point in the cluster into its neighborhood set,
as illustrated in Figure~\ref{fig:clusterneighborhoods}. When assigning a point $\mathbf{x}$ to its nearest cluster,
we only need to consider those clusters that contain $\mathbf{x}$ in their closures.
Typically a point belongs to a small number of cluster closures,
thus the number of candidate clusters are greatly reduced in the assignment step.

We evaluate our algorithm by complexity analysis,
the performance on clustering real data sets, and the performance of image retrieval applications with codebooks learned by clustering.
Our proposed algorithm achieves significant improvements compared to the state-of-the-art,
in both accuracy and running time.
When clustering a real data set of $1M$ $384$-dimensional GIST features into $10K$ clusters,
our algorithm converges more than $2.5$ faster than the state-of-the-art algorithms.
In the image retrieval application on a standard dataset,
our algorithm learns a codebook with $750K$ visual words
that outperforms the codebooks with $1M$ visual words learned
by other state-of-the-art algorithms --
even our codebook with $500K$ visual words is superior
over other codebooks with $1M$ visual words.

\section{Literature review}
Given a set of points $\{\mathbf{x}_1, \mathbf{x}_2, \cdots,
\mathbf{x}_n\}$, where each point is a $d$-dimensional vector,
$k$-means clustering aims to partition these $n$ points into $k$ ($k \leqslant n$)
groups, $\mathcal{G} = \{\mathcal{G}_1,
\mathcal{G}_2, \cdots, \mathcal{G}_k\}$, by minimizing the
within-cluster sum of squared distortions (WCSSD):
\begin{align}
J(\mathcal{C}, \mathcal{G}) = \sum\nolimits_{j=1}^k \sum\nolimits_{\mathbf{x}_i \in \mathcal{G}_j} \|\mathbf{x}_i - \mathbf{c}_j\|_2^2,
\label{eqn:WCSS}
\end{align}
where $\mathbf{c}_j$ is the center of cluster $\mathcal{G}_j$,
$\mathbf{c}_j = \frac{1}{|\mathcal{G}_j|} \sum_{\mathbf{x}_i \in \mathcal{G}_j}  \mathbf{x}_i$,
and $\mathcal{C} = \{\mathbf{c}_1, \cdots, \mathbf{c}_k\}$.
In the following,
we use {\em group} and {\em cluster} interchangeably.

\subsection{Lloyd algorithm}
Minimizing the objective function in Equation~\ref{eqn:WCSS} is
NP-hard in many cases~\cite{MahajanNV09}. Thus, various
heuristic algorithms are used in practice,
and $k$-means (or Lloyd's algorithm)~\cite{Forgy65,
Lloyd82, MacQueen67} is the most commonly used algorithm.
It starts from a set of $k$ cluster centers (obtained from priors or random initialization)
$\{\mathbf{c}_1^{(1)}, \cdots, \mathbf{c}_k^{(1)}\}$, and then
proceeds by alternating the following two steps:
\begin{itemize}
  \item \textbf{Assignment step:} Given the current set of
      $k$ cluster centers, $\mathcal{C}^{(t)} = \{\mathbf{c}_1^{(t)},
      \cdots, \mathbf{c}_k^{(t)}\}$, assign each point
      $\mathbf{x}_i$ to the cluster whose center is the closest to $\mathbf{x}_i$:
  \begin{align}
  z_i^{(t+1)} = \arg\min\nolimits_{j} \|\mathbf{x}_i - \mathbf{c}_j^{(t)}\|_2.
  \end{align}
  \item \textbf{Update step:} Update the points in each
      cluster, $\mathcal{G}_j^{(t+1)} = \{\mathbf{x}_i |
      z_i^{(t+1)} = j\}$, and compute the new center for each cluster,
      $\mathbf{c}_j^{(t+1)} =
      \frac{1}{|\mathcal{G}_j^{(t+1)}|} \sum_{\mathbf{x}_i
      \in \mathcal{G}_j^{(t+1)}}  \mathbf{x}_i$.
\end{itemize}
The computational complexity for the above
assignment step and the update step is $O(ndk)$ and $O(nd)$, respectively.
Various speedup algorithms have been developed
by making the complexity of the assignment step
less than the linear time (e.g., logarithmic time) with respect to
$n$ (the number of the data points),
$k$ (the number of clusters),
and
$d$ (the dimension of the data pint).
In the following,
we present a short review
mainly on handling large $n$ and $k$.

\subsection{Handling large data}
\mytextbf{Distance computation elimination.}
Various approaches have been proposed to speed up exact $k$-means.
An accelerated algorithm is proposed
by using the triangle inequality~\cite{Elkan03}
and keeping track of lower and upper
bounds for distances between points and centers
to avoid unnecessary distance calculations
but requires $O(k^2)$ extra storage,
rendering it impractical for a large number of clusters.

\mytextbf{Subsampling.}
An alternative solution to speed up $k$-means
is based on sub-sampling the data points.
One way is to run $k$-means over sub-sampled data points,
and then to directly assign the remaining points to the clusters.
An extension of the above solution is to optionally
add the remaining points incrementally,
and to rerun $k$-means to get a finer clustering.
The former scheme is not applicable in many applications.
As pointed in~\cite{PhilbinCISZ07},
it results in less accurate clustering
and lower performance in image retrieval applications.
The Coremeans algorithm~\cite{FrahlingS08} uses the latter scheme.
It begins with a coreset
and incrementally increases the size of the coreset.
As pointed out in~\cite{FrahlingS08},
Coremeans works well only for a small number of clusters.
Consequently, those methods are not suitable
for large-scale clustering problems,
especially for problems with a large number of clusters.

\mytextbf{Data organization.}
The approach in~\cite{KanungoMNPSW02} presents
a filtering algorithm.
It begins by storing the data points in a $k$-d tree
and maintains,
for each node of the tree,
a subset of candidate centers.
The candidates for each node are pruned or filtered,
as they propagate to the children,
which eliminates the computation time
by avoiding comparing each center with all the points.
But as this paper points out,
it works well only when the number of clusters is small.

In the community of document processing, Canopy clustering~\cite{McCallumNU00},
which is closely related to our approach,
first divides the data points
into many overlapping subsets (called canopies),
and clustering is performed
by measuring exact distances
only between points that occur within a common canopy.
This eliminates a lot of unnecessary distance computations.

\subsection{Handling large clusters}
\mytextbf{Hierarchical $k$-means.}
The hierarchical $k$-means (HKM) uses a clustering tree instead of flat $k$-means~\cite{NisterS06}
to reduce the number of clusters in each assignment step.
It first clusters the points into a small number (e.g., 10) of clusters,
then recursively divides each cluster
until a certain depth $h$ is reached.
The leaves in the resulted clustering tree are
considered to be the final clusters.
(For $h=6$, one obtains one million clusters.)

Suppose
that the data points associated with each node of the hierarchial tree
are divided
into a few (e.g., a
constant number $\bar{k}$, much smaller than $k$) subsets
(clusters).
In each recursion, each point can only be assigned to one
of the $\bar{k}$ clusters, and the depth of the recursions is
$O(\log n)$.
The computational cost is $O(n\log n)$ (ignoring the
small constant number $\bar{k}$).

\mytextbf{Approximate $k$-means.}
In~\cite{PhilbinCISZ07} approximate nearest neighbor (ANN) search replaces
the exact nearest neighbor (NN) search in the assignment step when searching
for the nearest cluster center for each point.
In particular, the current cluster centers
in each $k$-means iteration are organized
by a forest of $k$-d trees to perform an accelerated approximate NN search.
The cost of the assignment step is reduced to $O(k\log k + M n \log k) =
O(Mn\log k)$, with $M$ being the number of accessed nearest
cluster candidates in the $k$-d trees.
Refined-AKM (RAKM)~\cite{Philbin10a} further improves the convergence speed by
by enforcing constraints of non-increasing objective values during the iterations.
Both AKM and RAKM require a considerable overhead of constructing $k$-d trees
in each $k$-means iteration, thus a trade-off between the speed and
the accuracy of the nearest neighbor search has to be made.

\subsection{Others}
There are some other complementary works in improving $k$-means clustering.
In~\cite{Sculley10}, the update step is speeded up
by transforming a batch update to a mini-batch update.
The high-dimensional issue has also been addressed
by using dimension reduction,
e.g., random projections~\cite{BoutsidisZ10, FernB03a} and product quantization~\cite{JegouDS11}.

Object discovery and mining from spatially related images
is one topic that is related to image clustering~\cite{ChumM10, LiWZLF08, PhilbinZ08, RaguramWFL11, SimonSS07},
which also aims to cluster the images
so that each group contains the same object.
This is a potential application
of our scalable $k$-means algorithm.
In~\cite{WangWZTGL12,WangWZTGL13},
we introduce an algorithm
of clustering spatially-related images
based on the neighborhood graph.
The idea of constructing the neighborhood graph
is to adopt multiple spatial partition trees,
which is similar to the idea of this paper.

\section{$K$-means with cluster closures}
In this section,
we first introduce the proposed approach,
then give the analysis and discussions,
and finally present the implementation details.

\begin{figure}[t]
\sidecaption[t]
\includegraphics[width=0.6\linewidth, clip]{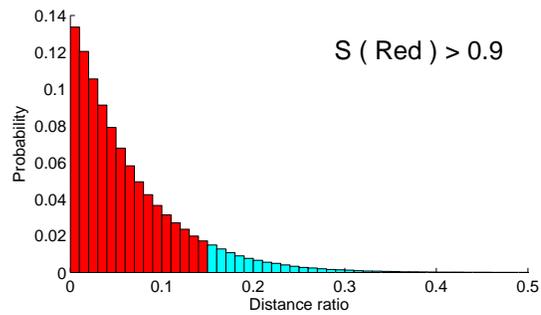}
\caption{The distribution of the distance ratio. It shows that
most active points have smaller distance ratio and lie near
some cluster boundaries} \label{fig:boundarypoints}
\end{figure}

\begin{figure}[b]
\sidecaption[t]
\includegraphics[width = .4\linewidth,clip]{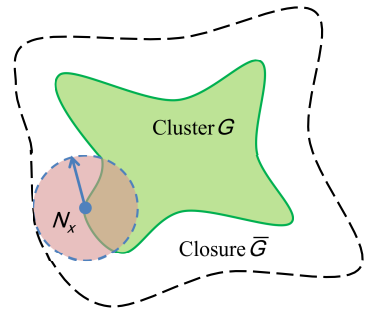}
\caption{Illustration of uniting neighborhoods
to obtain the closure. The black dash line indicates the closure of
cluster $\mathcal{G}$}\label{fig:clusterneighborhoods}
\end{figure}

\begin{figure}[t]
\sidecaption[t]
\includegraphics[width = .6\linewidth,clip]{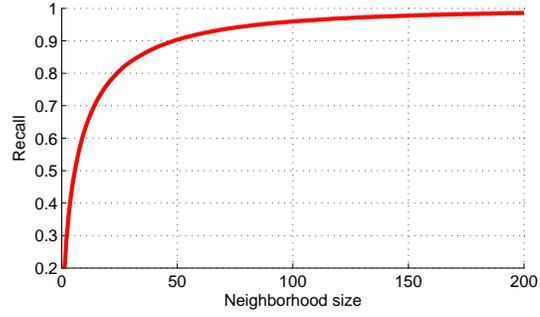}
\caption{The coverage of the active points by the closure w.r.t.
the neighborhood size. A neighborhood of size 50 has about
$90\%$ coverage}\label{fig:recallneighbor}
\end{figure}

\subsection{Approach}
\mytextbf{Active points.}
$K$-means clustering partitions the data space into Voronoi
cells -- each cell is a cluster and the cluster center is the
center of the cell. In the assignment step, each point $\mathbf{x}$
is assigned to its nearest cluster center. We call points that
change cluster assignments in an iteration \emph{active points}. In other
words, $\mathbf{x}$ changes cluster membership from the $i$-th cluster to
the $j$-th cluster because $d(\mathbf{x}, \mathbf{c}_j) < d(\mathbf{x}, \mathbf{c}_i)$,
where $d(\cdot)$ is the distance function.

We observe that \emph{active points are close to the boundary}
between $\mathbf{c}_j$ and $\mathbf{c}_i$. To verify this, we define
\emph{distance ratio} for an active point $\mathbf{x}$ as:
$ r(\mathbf{x}) = 1 - \frac{d(\mathbf{x},\mathbf{c}_j)}{d(\mathbf{x},\mathbf{c}_i)}$.
The distance ratio $r(\mathbf{x})$ is in the range of $(0,1]$,
since we only compute distance ratio for active points.
Smaller values of $r$ mean closer to the cluster boundaries.
Figure~\ref{fig:boundarypoints} shows the distribution of
distance ratios when clustering $1M$ GIST features
from the Tiny image data set (described in~\ref{subsec:datasets}) to $10K$ clusters.
We can see that most active points have small distance ratios,
e.g. more than $90\%$ of the active points have a distance ratio less than $0.15$
(shown in the red area),
and thus lie near to cluster boundaries.

During the assignment step,
we only need to identify the active points
and change their cluster memberships.
The above observation that active points lie close
to cell boundaries suggests a novel approach to
speed up the assignment step
by identifying active points around cell boundaries.

\mytextbf{Cluster closures.}
Assume for now that we have identified the neighborhood of
a given point $\mathbf{x}$,
a set of points
containing $\mathbf{x}$'s neighboring points and itself,
denoted by $\mathcal{N}_x$. We define the
{\em closure} of a cluster $\mathcal{G}$ as:
\begin{equation}\label{eq:group-closure}
  \bar{\mathcal{G}} = \bigcup\nolimits_{\mathbf{x} \in \mathcal{G}} \mathcal{N}_{x}.
\end{equation}
Figure~\ref{fig:clusterneighborhoods} illustrates the
relationship between the cluster, the
neighborhood points, and the closure.

If active points are on the cluster boundaries,
as we have observed,
then by increasing the neighborhood size $\mathcal{N}_x$,
the group closure $\bar{\mathcal{G}}$ will be accordingly
expanded to cover more active points that will be assigned to
this group $\mathcal{G}$ in the assignment step.
Figure~\ref{fig:recallneighbor} shows the recall
(of an active point being covered by the closure of its newly assigned cluster)
vs. the neighborhood size of $\mathcal{N}_x$
over the Tinyimage data set
describe in Section~\ref{subsec:datasets}.
Similar results are also observed in other data sets.
As we can see, with a neighborhood size as small as $50$,
about $90\%$ of the active points are covered
by the closures of the clusters to
which these active points will be re-assigned.

We now turn to the question of how to efficiently compute the
neighborhood $\mathcal{N}_x$ of a given point $\mathbf{x}$
used in Equation~\ref{eq:group-closure}. We propose an ensemble
approach using multiple random spatial partitions. A single
approximate neighborhood for each point can be derived from a
random partition (RP) tree~\cite{VermaKD09}, and the final
neighborhood is assembled by combining the results from multiple random
spatial partitions.
Suppose that a leaf node of a single RP tree,
contains a set of points $\mathcal{V} = \{\mathbf{x}_{j}\}$,
we consider all the points in $\mathcal{V}$ to be mutually
neighboring to each other. Thus the neighborhood of a point
$\mathbf{x}$ in the set $\mathcal{V}$ can be straightforwardly
computed by $\mathcal{N}_{x} = \mathcal{V}$.

Since RP trees are efficient to construct, the above
neighborhood computation is also efficient. While the group
closure from one single RP tree may miss some active points,
using multiple RP trees effectively handles this problem. We
simply unite the neighborhoods of $\mathbf{x}$ from all the RP trees:
\begin{displaymath}
 \mathcal{N}_{x} = \bigcup\nolimits_l \mathcal{V}_l.
\end{displaymath}
Here $\mathcal{V}_l$ is a set of points in the leaf
from the $l$-th RP tree that contains $\mathbf{x}$.
Note that a point $\mathbf{x}$ may belong to multiple group closures.
Also note that the neighborhood of a given point is computed only once.

\mytextbf{Fast assignment.}
With the group closures $\{\bar{\mathcal{G}}_j\}$ computed from
Equation~\ref{eq:group-closure}, the assignment step can be
done by verifying whether a point belonging to the closure
$\bar{\mathcal{G}}_j$ should indeed be assigned to the cluster
$\mathcal{G}_j$:
\begin{itemize}
\item \textbf{Initialization step:} Initialize the distance
    array $D[1:n]$ by assigning an positive infinity value to each
    entry.
  \item \textbf{Closure-based assignment:} \\
~~~~~~For each cluster closure $\{\bar{\mathcal{G}}_j\}$:\\
~~~~~~~~~~~~~~~~~~~~~~\hspace{4mm} For each point
$\mathbf{x}_{i}^s \in \bar{\mathcal{G}}_j, s=1,2,...,|\bar{\mathcal{G}}_j|$:
\begin{align}
\text{~~if:~~~~~~} & \|\mathbf{x}_i^s - \mathbf{c}_j^{(t)}\|_2^2 < D[i], \nonumber\\
\text{then:~~~~~~} & z_{i}^{(t+1)} = j, \nonumber\\
                   & D[i] = \|\mathbf{x}_{i}^s - \mathbf{c}_j^{(t)}\|_2^2. \nonumber
\end{align}
Here $\mathbf{c}_j^{(t)}$ is the cluster center of
$\mathcal{G}_j$ at the $t$-th iteration,
$i$ is the global index for $\mathbf{x}$ and
$s$ is the index into $\bar{\mathcal{G}}_j$ for point $\mathbf{x}_i$.
\end{itemize}


In the assignment step, we only need to compute the distance
from the center of a cluster to each point in the cluster
closure. A point typically belongs to a small number of cluster
closures.  Thus, instead of computing the distances from a point $\mathbf{x}$ to all cluster
centers in exact $k$-means, or constructing $k$-d trees of all
cluster centers at each iteration to find the approximate nearest
cluster center, we only need to compute the distance from $\mathbf{x}$ to a small
number of cluster centers whose cluster closures contain $\mathbf{x}$,
resulting in a significant reduction in computational cost.
Moreover, the fact that active points are close to cluster
boundaries is the worst case for $k$-d trees to find the nearest
neighbor. On the contrary, such a fact is advantageous for our
algorithm.

\subsection{Analysis}
\mytextbf{Convergence.} The following shows that our
algorithm always converges. Since the objective function
$J(\mathcal{C}, \mathcal{G})$ is lower-bounded, the convergence
can be guaranteed if the objective value does
not increase at each iterative step.

\begin{theorem}[Non-increase]
The value of the objective function does not increase at each iterative step, i.e.,
\begin{equation}\label{eq:theo1}
J(\mathcal{C}^{(t + 1)}, \mathcal{G}^{(t + 1)})
\leqslant  J(\mathcal{C}^{(t)}, \mathcal{G}^{(t)}).
\end{equation}
\label{lemma:convergence}
\end{theorem}
\begin{proof}
\smartqed
In the assignment step for the $(t+1)$-th iteration,
$\{{c}^{(t)}_k\}$ computed from the $t$-th iteration
are cluster candidates.
$\mathbf{x}_i$ would change its cluster membership only if it finds a closer cluster center, thus we have $\| \mathbf{x}_i -
\mathbf{c}^{(t)}_{z_i^{(t + 1)}} \|_2 \leqslant \| \mathbf{x}_i
- \mathbf{c}^{(t)}_{z_i^{(t)}} \|_2$, and Equation~\ref{eq:theo1} holds for the assignment step.

In the update step, the cluster center will then be update based on the new point assignments. We now show that this update will not increase the within-cluster sum
of squared distortions, or in a more general form:
\begin{align}\label{eqn:non-increase1}
\sum\nolimits_{\mathbf{x} \in \mathcal{G}_j} \| \mathbf{x} - \bar{\mathbf{c}}_j\|_2^2
\leqslant \sum\nolimits_{\mathbf{x} \in \mathcal{G}_j} \| \mathbf{x} - \mathbf{c}\|_2^2,
\end{align}
where $\bar{\mathbf{c}}_j$ is the $j$-th updated cluster center $\bar{\mathbf{c}}_j =
\frac{1}{|\mathcal{G}_j|}\sum_{\mathbf{x} \in
\mathcal{G}_j}\mathbf{x}$, and $\mathbf{c}$ is an arbitrary point in the data space.
Equation~\ref{eqn:non-increase1} can be verified by the following:
\begin{align}
&~\sum\nolimits_{\mathbf{x} \in \mathcal{G}_j} \| \mathbf{x} - \mathbf{c}\|_2^2 \nonumber \\
= &~\sum\nolimits_{\mathbf{x} \in \mathcal{G}_j} \| (\mathbf{x} - \bar{\mathbf{c}}_j) + (\bar{\mathbf{c}}_j - \mathbf{c})\|_2^2 \nonumber \\
= &~\sum\nolimits_{\mathbf{x} \in \mathcal{G}_j} \| \mathbf{x} - \bar{\mathbf{c}}_j\|_2^2 + 2 (\bar{\mathbf{c}}_j - \mathbf{c})^T \sum\nolimits_{\mathbf{x} \in \mathcal{G}_j} (\mathbf{x} - \bar{\mathbf{c}}_j) \nonumber \\
&~+ |\mathcal{G}_j| \|\bar{\mathbf{c}}_j - \mathbf{c}\|_2^2 \nonumber \\
= &~ \sum\nolimits_{\mathbf{x} \in \mathcal{G}_j} \| \mathbf{x} - \bar{\mathbf{c}}_j\|_2^2 + |\mathcal{G}_j| \|\bar{\mathbf{c}}_j - \mathbf{c}\|_2^2 \nonumber \\
\geqslant&~ \sum\nolimits_{\mathbf{x} \in \mathcal{G}_j} \| \mathbf{x} - \bar{\mathbf{c}}_j\|_2^2.
\end{align}
Thus Equation~\ref{eq:theo1} holds for the update step.
\qed
\end{proof}

\mytextbf{Accuracy.} Our algorithm obtains the same result
as the exact Lloyd's algorithm
if the closures of the
clusters are large enough,
in such a way all the points that would have been assigned to the $j$-th cluster
when using the Lloyd's algorithm belong to the cluster closure $\bar{\mathcal{G}}_j$.
However, it should be noted that this condition is sufficient but not necessary. In practice,
even with a small neighborhood,
our approach often obtains results similar to using the exact Lloyd's
algorithm. The reason is that the missing points, which should have been assigned
to the current cluster at the current iteration but are missed,
are close to the cluster boundary
thus likely to appear
in the closure of the new clusters
updated by the current iteration.
As a result, these missing points are very likely
to be correctly\footnote{``Correctly" w.r.t. assignments if produced by Lloyd's algorithm.} assigned in the next iteration.

\mytextbf{Complexity.}
Consider a point $\mathbf{x}_i$ and its neighborhood $\mathcal{N}_{x_i}$, the
possible groups that may absorb $\mathbf{x}_i$ are
$\tilde{\mathcal{G}}_{x_i} = \{\mathcal{G}_j | ~\exists~\mathbf{x}_j \mbox{ s.t. } \mathbf{x}_j\in
\mathcal{G}_j \mbox{ and } \mathbf{x}_j \in \mathcal{N}_{x_i}\}$. As a result, we have
$|\tilde{\mathcal{G}}_{x_i}| \leqslant |\mathcal{N}_{x_i}|$. In our implementation, we use balanced
random bi-partition trees, with each leaf node containing
$c$  points ($c$ is a small number). Suppose we use $m$ random
partition trees. Then the neighborhood
size of a point will not be larger than $M = cm$. As a result,
the complexity of the closure-based assignment step is $O(nM)$.

For the complexity of constructing trees,
our approach constructs a RP-tree in $O(n\log n)$
and AKM costs $O(k\log k)$ to build a $k$d-tree.
However, our approach only needs a small number
(typically $10$ in our clustering experiments) of trees
through all iterations, but AKM
requires constructing a number
(e.g., $8$ in~\cite{PhilbinCISZ07})
of trees in each iteration,
which makes the total cost more expensive.

\subsection{Discussion}
We present the comparison of our approach
with most relevant three algorithms,
Canopy clustering,
approximate $k$-means,
and
hierarchical $k$-means.

\mytextbf{Versus Canopy clustering.}
Canopy clustering, however, suffers from the canopy creation
whose cost is high for visual features.
More importantly,
it is non-trivial (1) to define a meaningful and efficient approximate distance function for visual data,
and (2) to tune the parameters for computing the canopy,
both of which are crucial to
the effectiveness and efficiency of Canopy clustering.
In contrast,
our approach is simpler and more efficient
because random partitions can be created
with a cost of only $O(n\log n)$.
Moreover, our method can adaptively update cluster member candidates,
in contrast to static canopies in~\cite{McCallumNU00}.

\mytextbf{Versus AKM.}
The advantages of the proposed approach
over AKM
are summarized as follows.
First,
the computational complexity of assigning
a new cluster to a point in our approach
is only $O(1)$,
while the complexity is $O(\log k)$ for AKM or RAKM.
The second advantage is
that we only need to organize the data points once
as the data points do not change during the iterations,
in contrast to AKM or RAKM that needs to construct the $k$-d trees
at each iteration as the cluster centers change from iteration to iteration.
Last,
It is shown that active points (points near cluster boundaries) present the worst case for ANN search
(used in AKM)
to return their accurate nearest neighbors.
In contrast, our approach is able to identify active points efficiently
and makes more accurate cluster assignment for active points
without the shortcoming in AKM.

\mytextbf{Versus HKM.}
As shown before,
HKM takes less time cost than AKM and our approach.
However,
its cluster accuracy is not as good
as HKM and our approach.
This is because
when assigning a point to a cluster (e.g., quantizing a feature descriptor)
in HKM,
it is possible that an error could be committed at a higher level of the tree,
leading to a sub-optimal cluster assignment and thus sub-optimal quantization.

\subsection{Implementation details}
\label{sec:implementation}

The random partition tree used for creating cluster closures
is a binary tree structure
that is formed by recursively splitting the space and
aims to organize the data points in a hierarchical manner.
Each node of the tree is associated
with a region in the space, called a cell.
These cells define a hierarchical decomposition of the space.
The root node $r$ is associated with the whole set of data points $\mathcal{X}$.
Each internal node $v$ is associated with a subset of data points $\mathcal{X}_v$ that lie
in the cell of the node.
It has two child nodes $\operatorname{left}(v)$ and $\operatorname{right}(v)$,
which correspond to two disjoint subsets of data points
$\mathcal{X}_{\operatorname{left}(v)}$ and $\mathcal{X}_{\operatorname{right}(v)}$.
The leaf node $l$ may be associated with a subset of data points
or only contain a single point.
In the implementation,
we use a random principal direction
to form the partition hyperplane to split the data points
into two subsets.
The principal directions are obtained by using
principal component analysis (PCA). To generate random
principal directions, rather than computing the principle
direction from the whole subset of points, we compute
the principal direction over the points randomly sampled
from each subset.
In our implementation, the principle
direction is computed by the Lanczos algorithm~\cite{Lanczos50}.

We use an adaptive scheme that
incrementally creates random partitions to automatically
expand the group closures on demand. At the beginning of our algorithm,
we only create one random partition tree. After each iteration,
we compute the reduction rate of the within-cluster sum of squared
distortions.
If the reduction rate in successive iterations is
smaller than a predefined threshold,
a new random partition tree is added to
expand points' neighborhood thus group closures.
We compare the adaptive neighborhood scheme
to a static one that computes
the neighborhoods altogether at the beginning (called static neighborhoods). As shown in
Figure~\ref{fig:adaptiverandomdivisions}, we can see that
the adaptive neighborhood scheme performs better in all the
iterations and hence
is adopted in the later comparison experiments.

The closure-based assignment step can be
implemented in another equivalent way. For each point
$\mathbf{x}$, we first identify the candidate centers by
checking the cluster memberships $\mathcal{Z}_x$ of the points within the neighborhood of $\mathbf{x}$. Here
$\mathcal{Z}_x = \{z(\mathbf{y}) \;|\; \mathbf{y} \in \mathcal{N}_x \}$, and $z(\mathbf{y})$ is the cluster membership of point $\mathbf{y}$. Then the best cluster
candidate for $\mathbf{x}$ can be found by checking the clusters
$\{\mathbf{c}_j \;|\; j \in \mathcal{Z}_x\}$.
In this equivalent implementation, the assignments are computed independently and can be naturally
parallelized. The update step computes the mean for each
the cluster independently, which can be naturally parallelized as well.
Thus, our algorithm can be easily parallelized.
We show the clustering performance with the parallel implementation (using multiple threads on multi-core CPUs) in Figure~\ref{fig:multithread}.

\begin{figure}[t]
\sidecaption[t]
\includegraphics[width = .45\linewidth, clip]{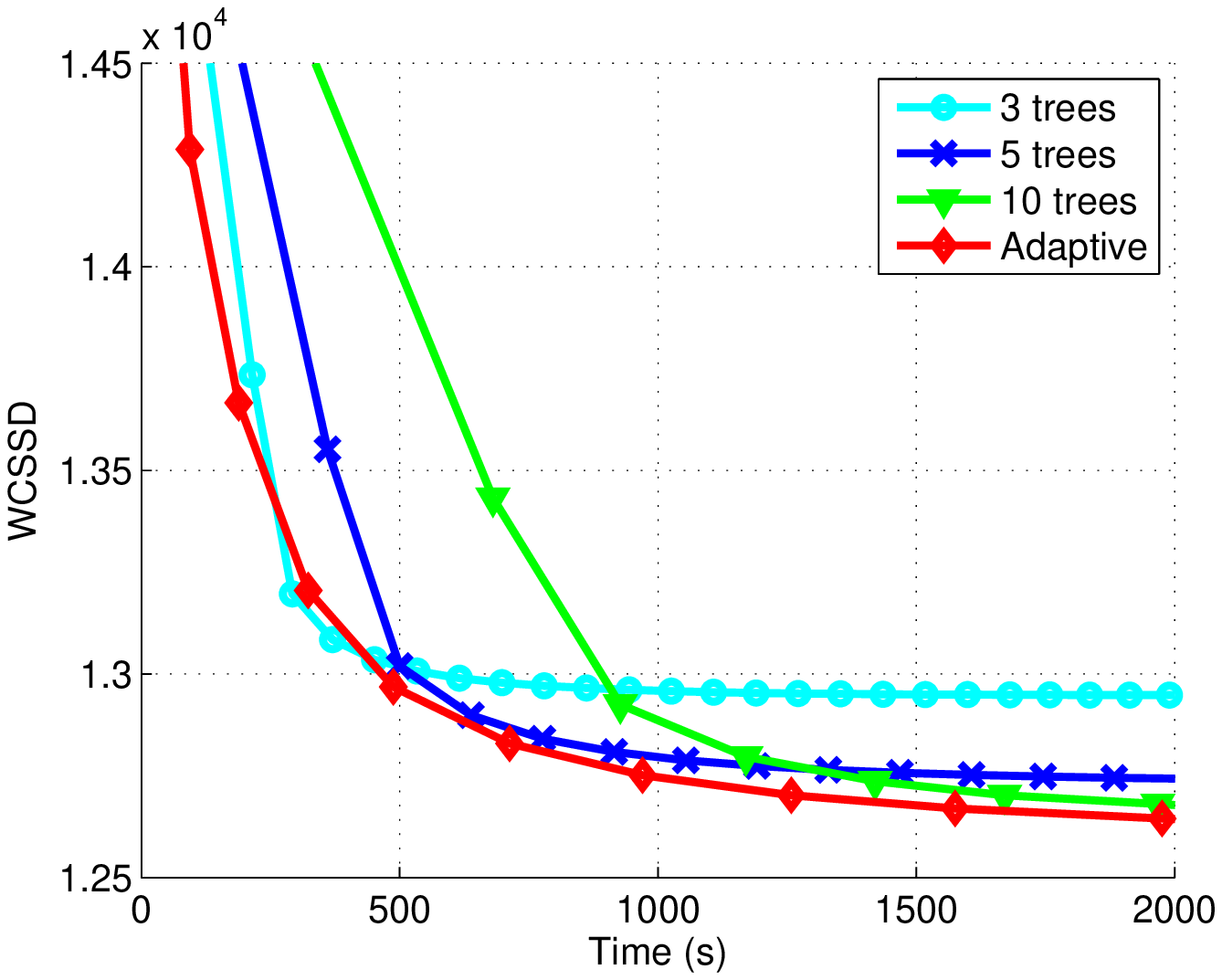}
\caption{Clustering performance with adaptive vs. static neighborhoods}
\label{fig:adaptiverandomdivisions}
\end{figure}

\begin{figure}[t]
\sidecaption[t]
\includegraphics[width = .45\linewidth, clip]{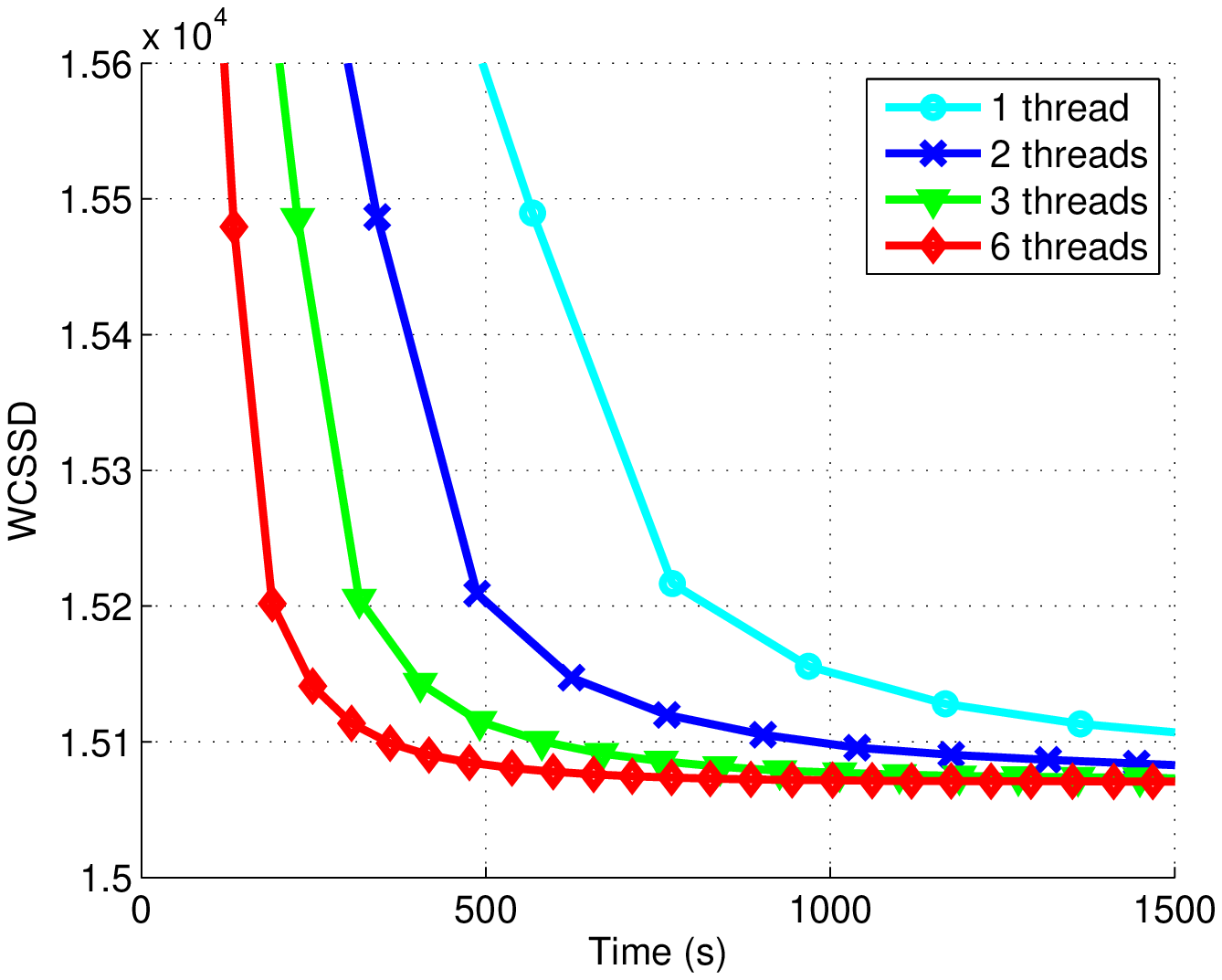}
\caption{Clustering performance with different numbers of threads}
\label{fig:multithread}
\end{figure}

\def \plotwidth {.3}
\begin{figure*}[t]
\centering
\includegraphics[width = \plotwidth\linewidth]{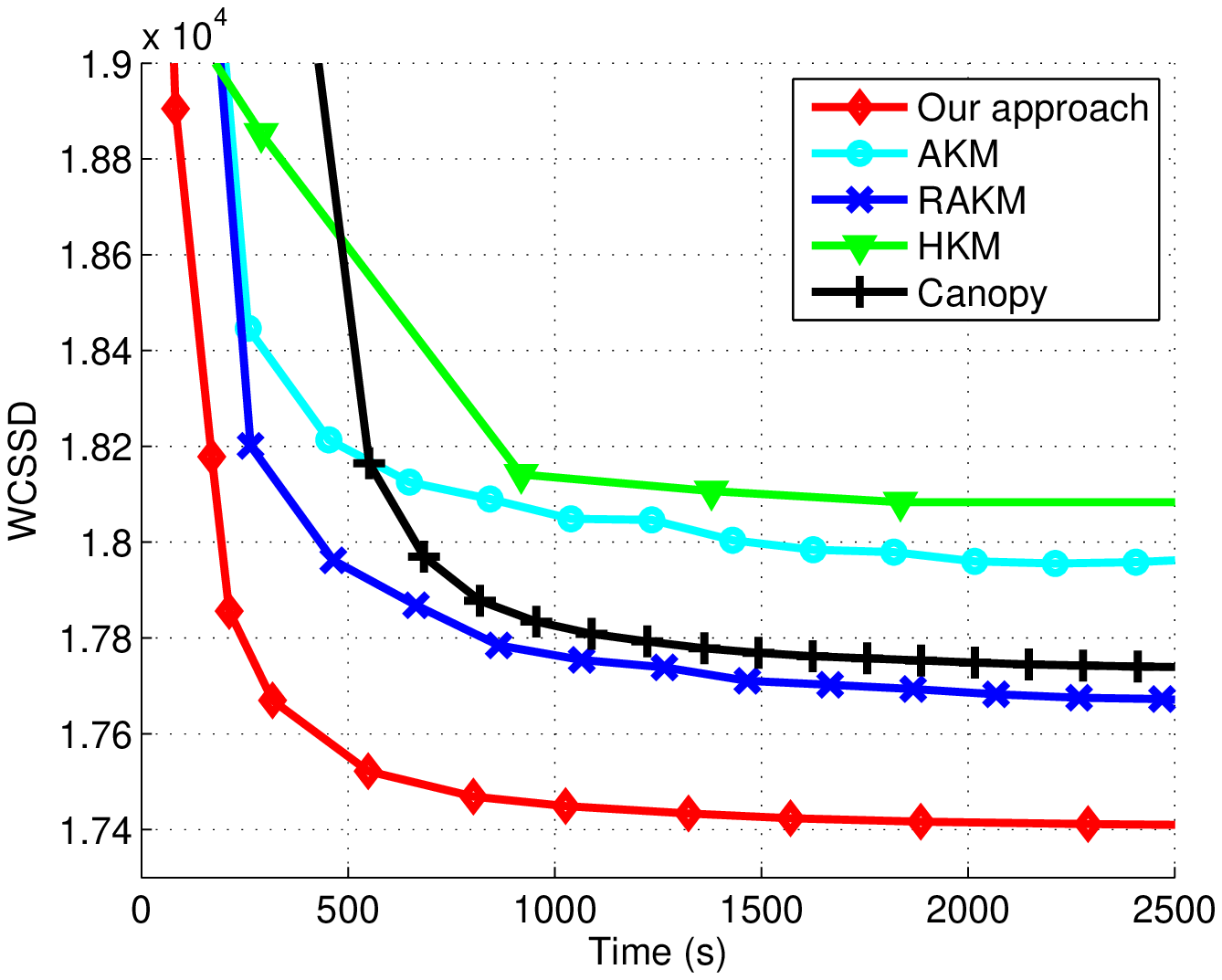}
~~~~\includegraphics[width = \plotwidth\linewidth]{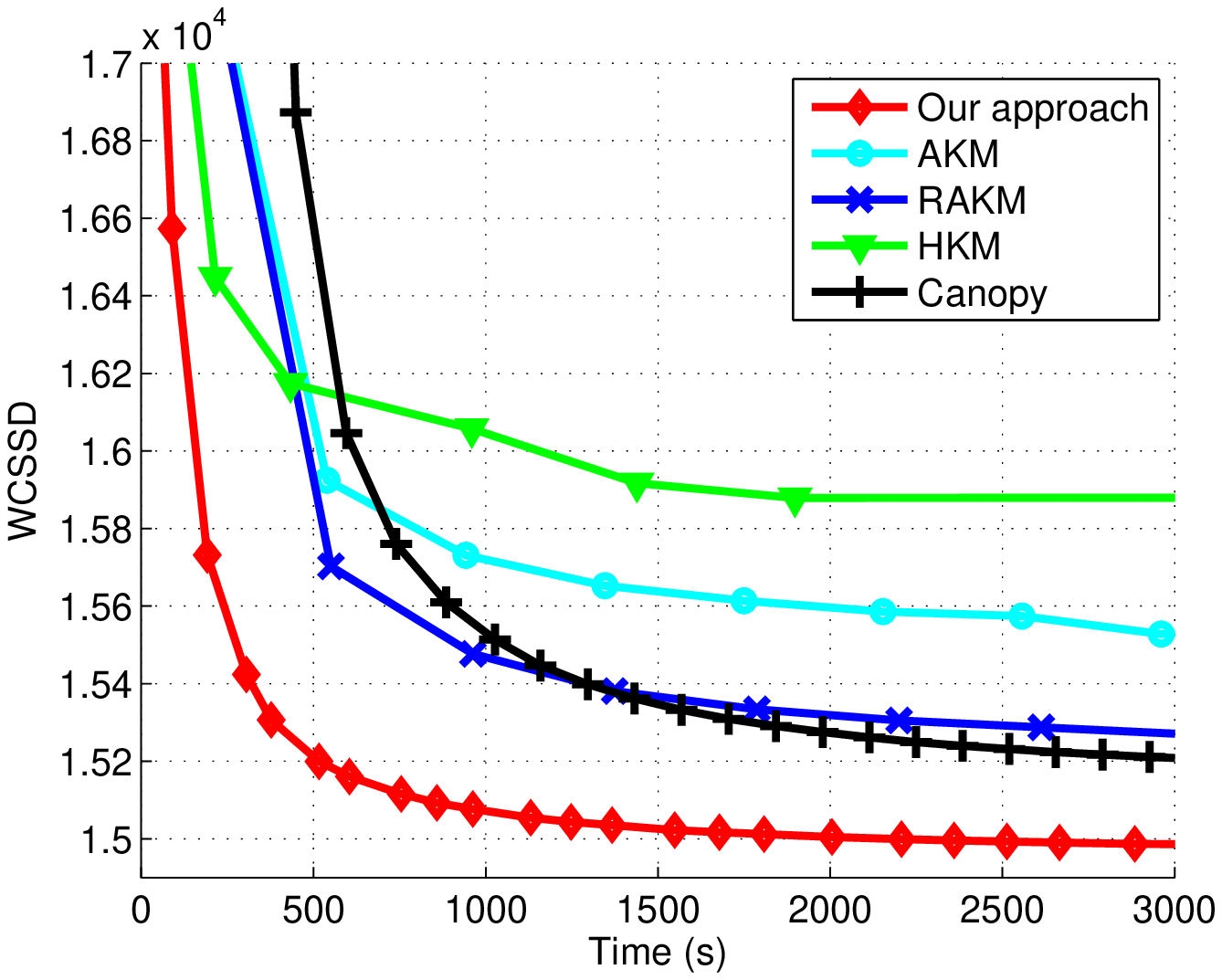}
~~~~\includegraphics[width = \plotwidth\linewidth]{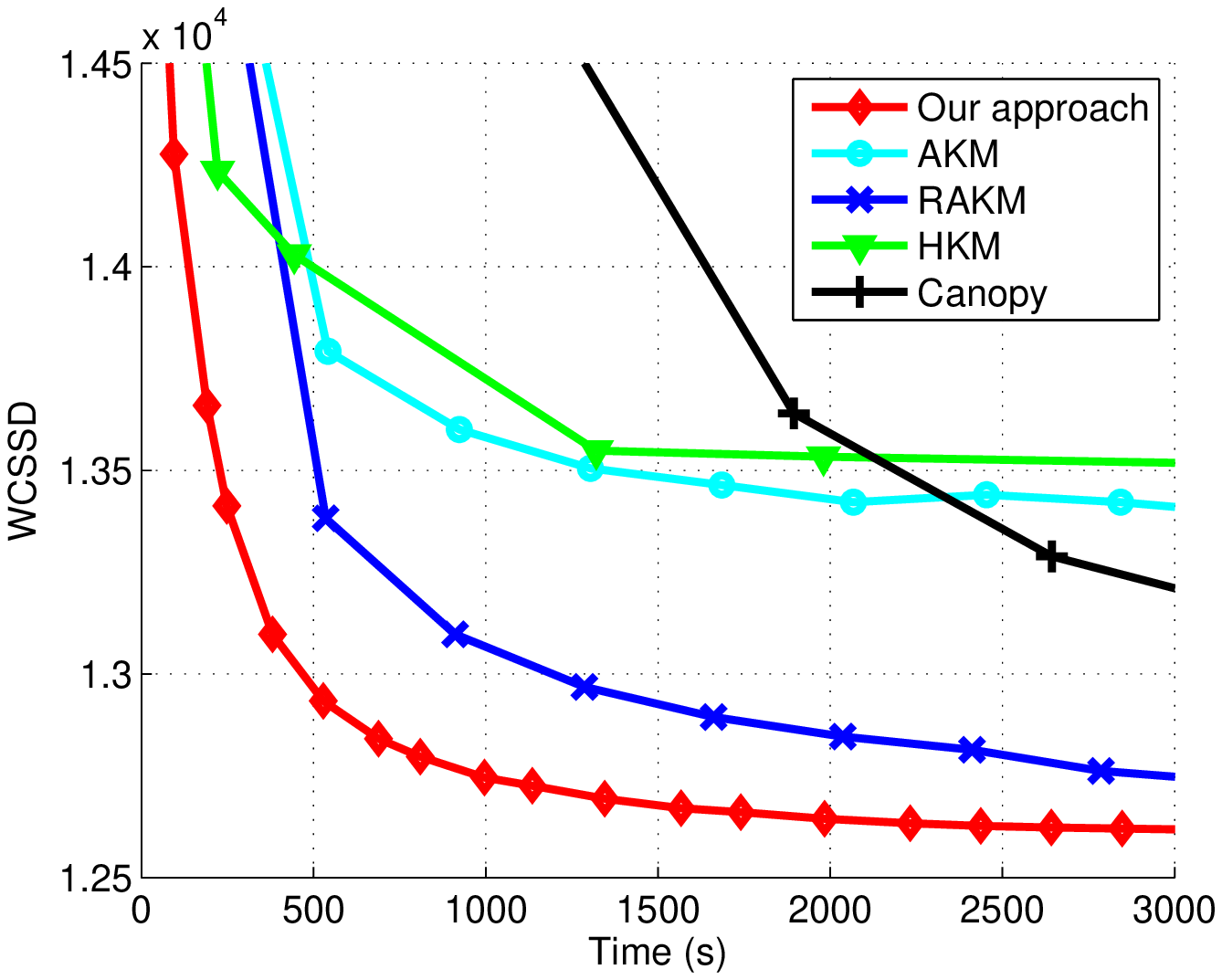}\\
\includegraphics[width = \plotwidth\linewidth]{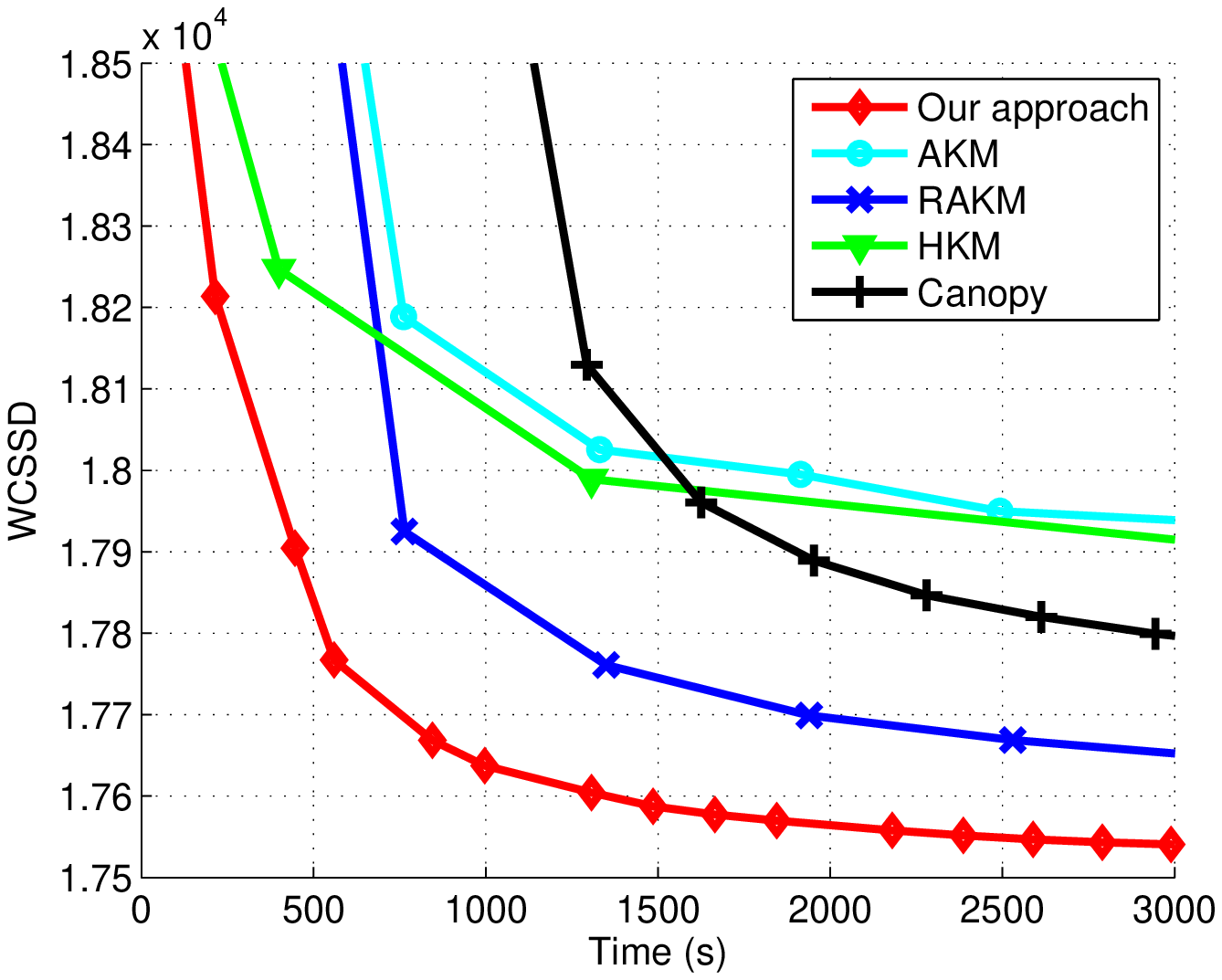}
~~~~\includegraphics[width = \plotwidth\linewidth]{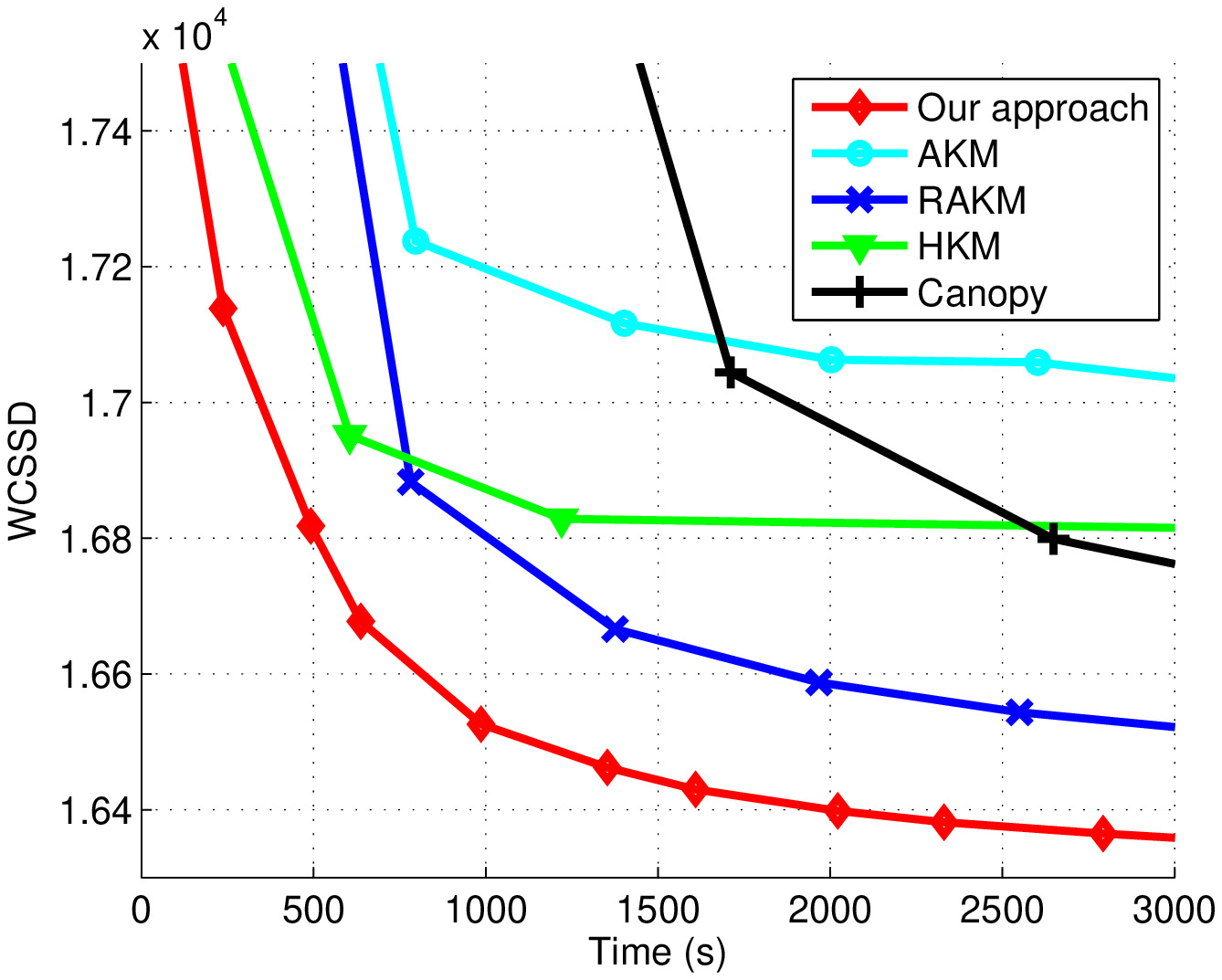}
~~~~\includegraphics[width = \plotwidth\linewidth]{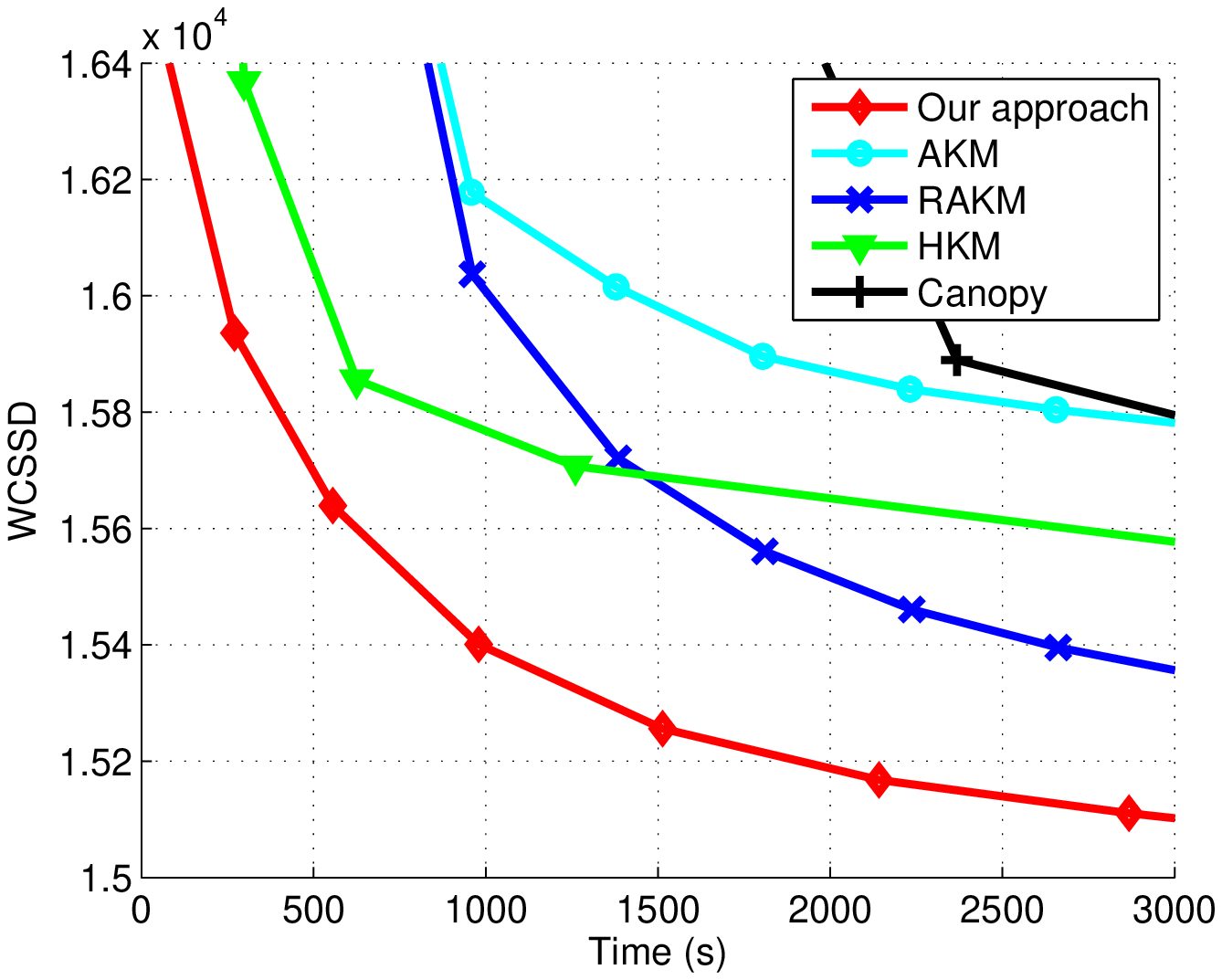}
\caption{Clustering performance in terms of within-cluster sum of squared distortions (WCSSD) vs. time.
The first row are the results of clustering 1M SIFT dataset into 0.5K, 2K and 10K clusters, respectively. The second row are results on 1M tiny image dataset}
\label{fig:WCSStime}
\end{figure*}

\begin{figure*}[t]
\centering
\subfigure[]{\includegraphics[width = \plotwidth\linewidth]{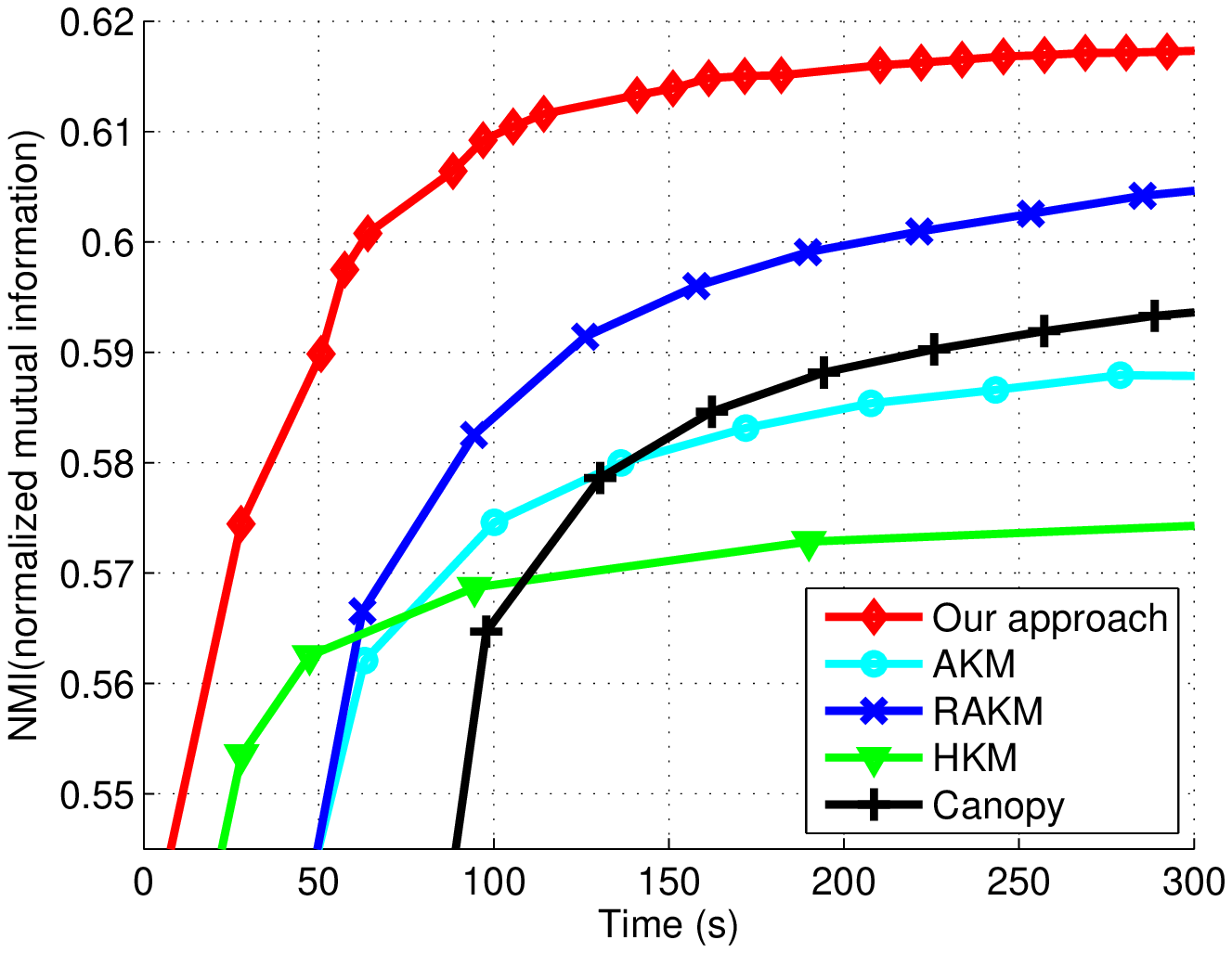}}
~~~~\subfigure[]{\includegraphics[width = \plotwidth\linewidth]{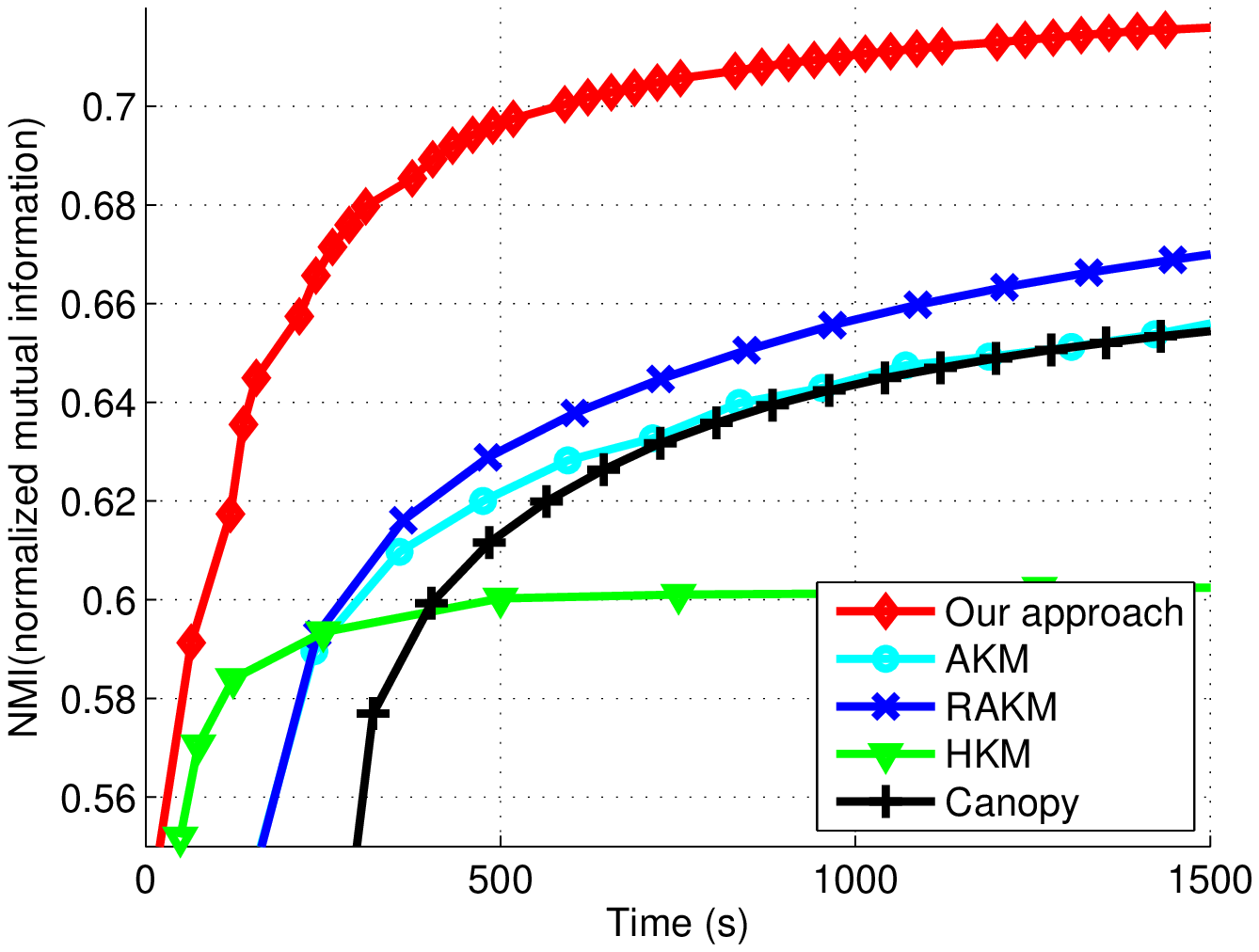}}
~~~~\subfigure[]{\includegraphics[width = \plotwidth\linewidth]{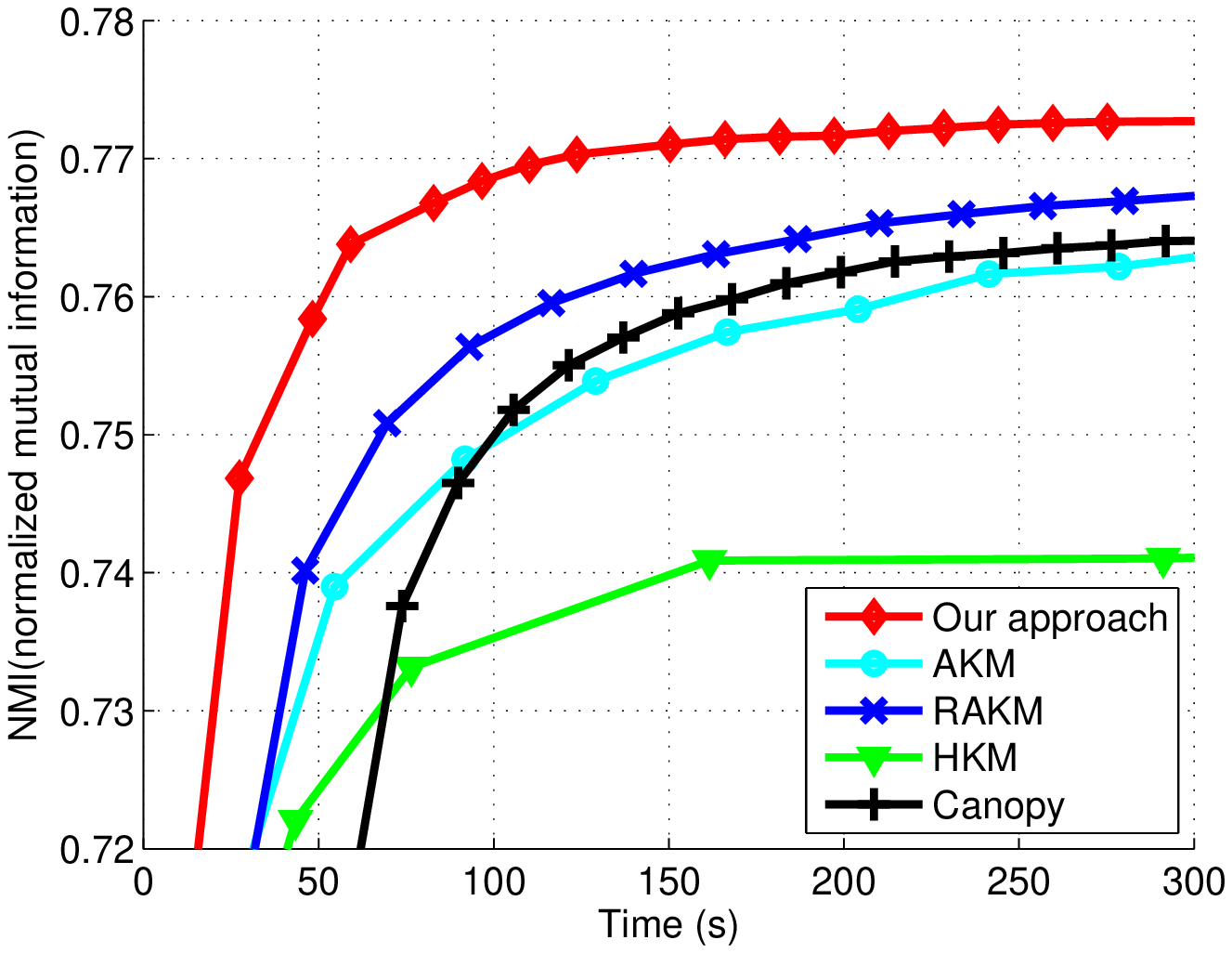}}
\caption{ Clustering performance in terms of normalized mutual information (NMI) vs. time, on the dataset of
(a) 200K tiny images, (b) 500K tiny images, and (c) 200K shopping images}
\label{fig:NMItime}
\end{figure*}

\section{Experiments}

\subsection{Data sets}
\label{subsec:datasets}

\mytextbf{SIFT.} The SIFT features are collected from the Caltech 101 data set~\cite{FeiFP04}.
We extract maximally stable extremal regions for each image,
and compute a $128$-dimensional SIFT feature for each region.
We randomly sample $1$ million features to form this data set.

\mytextbf{Tiny images.} We generate three data sets sampled from the tiny images~\cite{TorralbaFF08}:
$1M$ tiny images, $200K$ tiny images,
and $500K$ tiny images.
The $1M$ tiny images are randomly sampled
without using category (tag) information.
We sample $1K$ ($1.25K$) tags from the tiny images
and sample about $200$ ($400$) images for each tag,
forming $200K$ ($500K$) images.
We use a $384$-dimensional GIST feature to represent each image.

\mytextbf{Shopping images.} We collect about $5M$ shopping images from the Internet.
Each image is associated with a tag
to indicate its category.
We sample $1K$ tags and sample $200$ images for each tag
to form the $200K$ image set.
We use a $576$-dimensional HOG feature to represent each image.

\mytextbf{Oxford $5K$.} This data set~\cite{PhilbinCISZ07}
consists of $5062$ high resolution images of $11$ Oxford landmarks.
The collection has been manually annotated to
generate a comprehensive ground truth for $11$ different landmarks,
each represented by $5$ possible queries.
This gives a set of $55$ queries over which an object retrieval system can be evaluated.
The images,
the SIFT features,
and the ground truth labeling of this data set is publicly available\footnote{ \textup{http://www.robots.ox.ac.uk/\~{}vgg/data/oxbuildings/index.html}}.
This data set and the next data set will be used to
demonstrate the application of our approach to object retrieval.

\mytextbf{Ukbench $10K$.} This data set is from the Recognition Benchmark introduced in~\cite{NisterS06}.
It consists of $10200$ images split into
four-image groups, each of the same scene/object taken at different viewpoints.
The data set, the SIFT descriptors,
and the ground truth is publicly available\footnote{\textup{http://www.vis.uky.edu/\~{}stewe/ukbench/}}.

\subsection{Evaluation metric}

We use two metrics to evaluate the performance of various clustering algorithms,
the within-cluster sum of squared distortions (WCSSD) which is the objective value defined by Equation~\ref{eqn:WCSS},
and the normalized mutual information (NMI) which is widely used for clustering evaluation.
NMI requires the ground truth of cluster assignments $\mathcal{G}$ for points in the data set.
Given a clustering result $\mathcal{X}$, NMI is defined
by $\operatorname{NMI}(\mathcal{G}, \mathcal{X}) = \frac{I(\mathcal{G}, \mathcal{X})}{\sqrt{H(\mathcal{G})H(\mathcal{X})}}$,
where $I(\mathcal{G}, \mathcal{X})$ is the mutual information of $\mathcal{G}$ and $\mathcal{X}$
and $H(\cdot)$ is the entropy.

In object retrieval, image feature descriptors are quantized into visual words using codebooks.
A codebook of high quality will result in less quantization errors and more repeatable quantization results,
thus leading to a better retrieval performance.
We apply various clustering algorithms to constructing visual codebooks for object retrieval.
By fixing all the other components and parameters in our retrieval system except the codebook,
the retrieval performance is an indicator of the quality of the codebook.
For the Oxford $5K$ dataset, we follow~\cite{PhilbinCISZ07} to use mean average precision (mAP)
to evaluate the retrieved images.
For the ukbench $10K$ dataset, the retrieval performance is measured
by the average number of relevant images in the top $4$ retrieved images,
ranging from $0$ to $4$.

\subsection{Clustering performance comparison}

We compare our proposed clustering algorithm with four approximate $k$-means algorithms,
namely hierarchial $k$-means (HKM), approximate $k$-means (AKM), refined approximate $k$-means (RAKM)
and Canopy algorithm.
The exact Lloyd's is much less efficient
and prohibitively costly for large data sets,
so we do not report its results.
We use the implementation of HKM available from~\cite{MujaL09},
and the public release of AKM\footnote{\textup{http://www.robots.ox.ac.uk/\~{}vgg/software/fastcluster/}}.
The RAKM is modified from the above AKM release.
For Canopy algorithm,
we conduct principal component analysis over the features
to project them to a lower-dimensional subspace
to achieve a fast canopy construction.
For a fair comparison,
we initialize the cluster assignment
by a random partition tree in all algorithms except HKM
The time costs for constructing trees or other initialization
are all included in the comparisons.
All algorithms are run on a 2.66GHz desktop PC using a single thread.

Figure~\ref{fig:WCSStime} shows the clustering performance in terms of WCSSD vs. time.
The experiments are performed on two data sets,
the $1M$ $128$-dimensional SIFT data set and
the $1M$ $384$-dimensional tiny image data set, respectively.
The results are shown for different number of clusters, ranging from $500$ to $10K$.
Our approach consistently outperforms the other four approximate $k$-means algorithms -- it converges faster to a smaller objective value.

Figure~\ref{fig:NMItime} shows the clustering results in terms of NMI vs. time.
We use three labeled datasets,
the $200K$ tiny images, the $500K$ tiny images and the $200K$ shopping images.
Consistent with the WCSSD comparison results, our proposed algorithm is superior to the other four clustering algorithms.

We also show the qualitative clustering results
of our algorithm.
Figure~\ref{fig:ProductimageVisual} shows some examples
of the clustering results
over the $200K$ shopping images.
Figure~\ref{fig:TinyimageVisual} shows some examples
of the clustering results
over the $500K$ tiny images.
The first $3$ clusters are examples of similar objects,
the second $3$ clusters are examples of similar texture images,
and the last cluster are an example of similar sceneries.

\begin{figure}
\centering
\subfigure[]
{\label{rajeev1}
\includegraphics[width = 0.95\linewidth, clip]{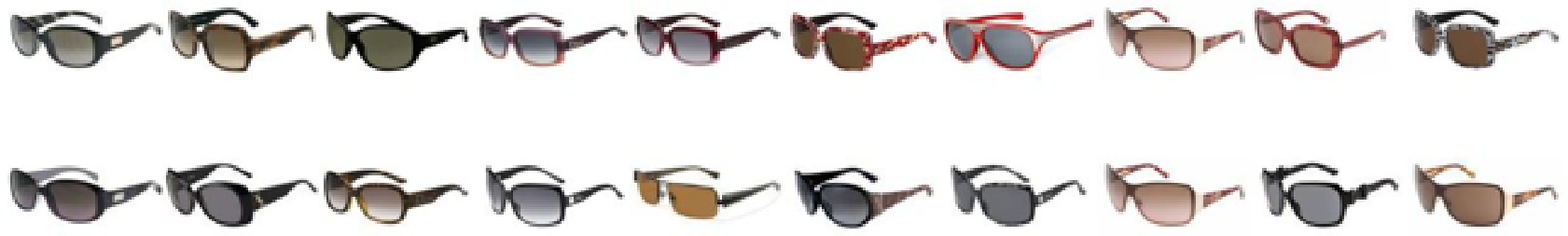}}\\

\subfigure[]
{\label{rajeev2}
\includegraphics[width = 0.95\linewidth, clip]{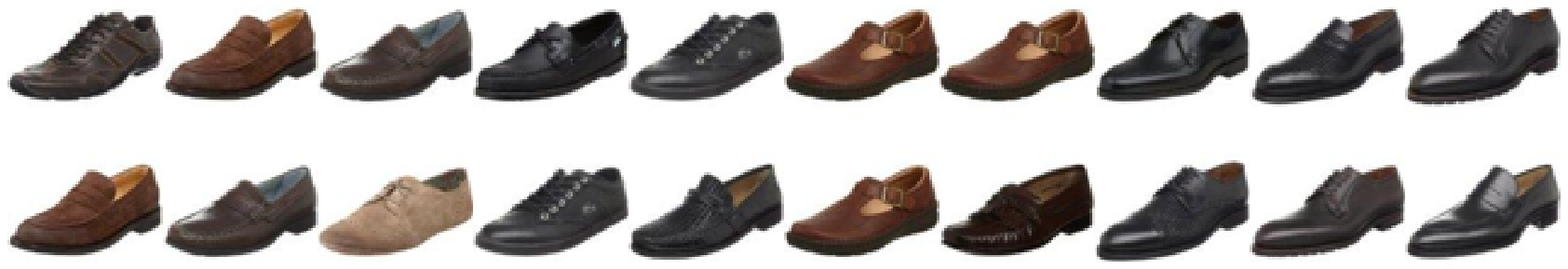}}\\

\subfigure[]
{\label{rajeev3}
\includegraphics[width = 0.95\linewidth, clip]{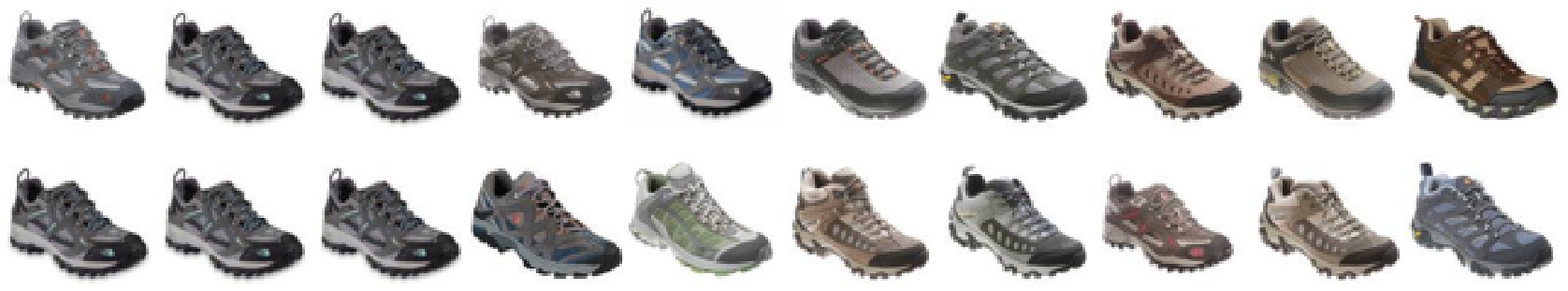}}\\

\subfigure[]
{\label{rajeev4}
\includegraphics[width = 0.95\linewidth, clip]{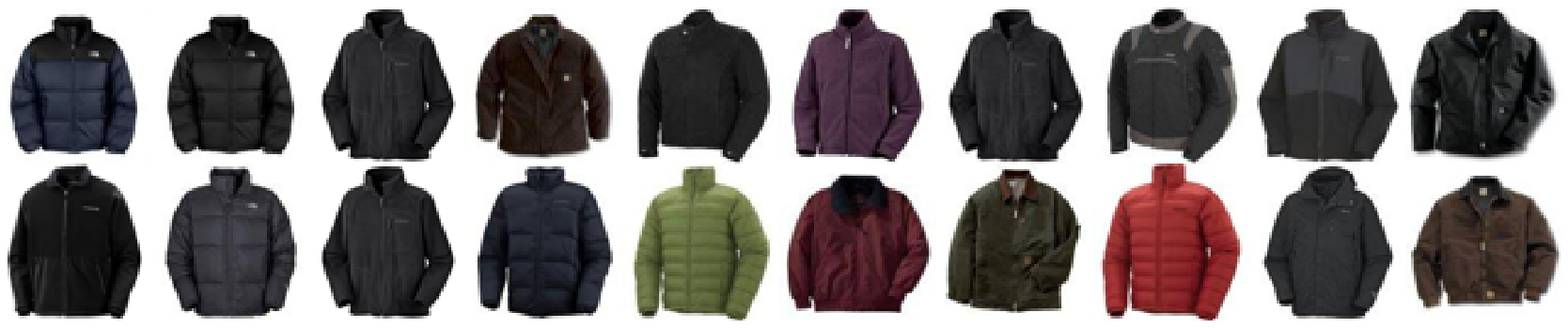}}\\

\subfigure[]
{\label{rajeev5}
\includegraphics[width = 0.95\linewidth, clip]{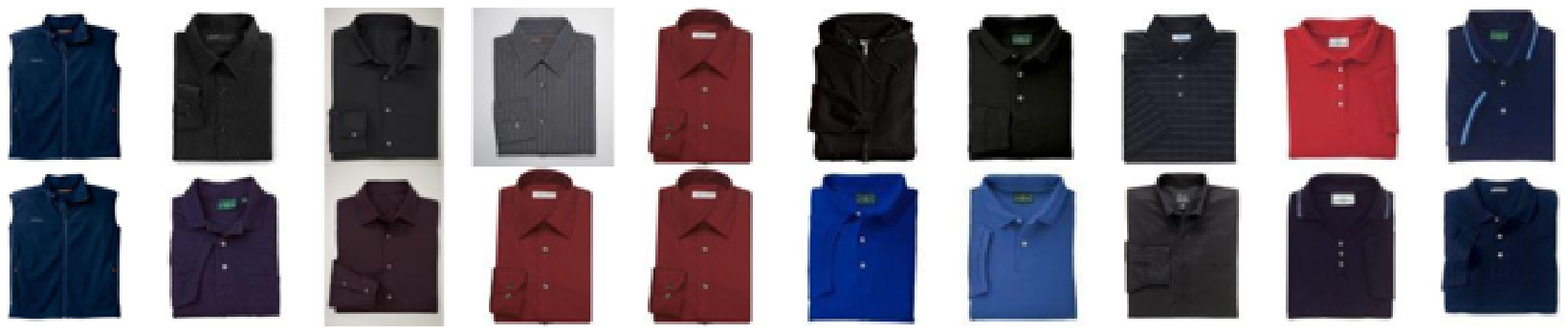}}\\

\subfigure[]
{\label{rajeev6}
\includegraphics[width = 0.95\linewidth, clip]{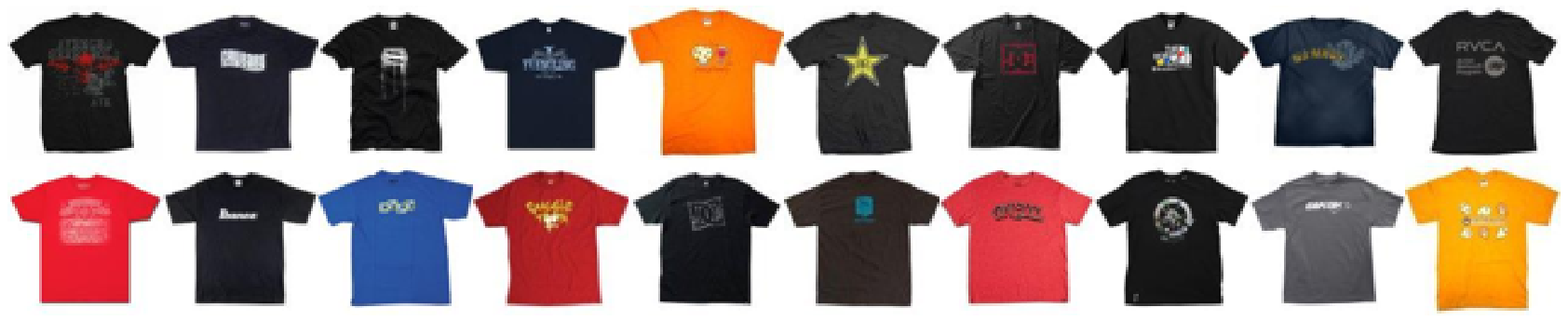}}\\
\caption{Clustering results over the $200K$ shopping images:
each cluster example is represented by two rows of images
which are randomly picked out from the cluster
}
\label{fig:ProductimageVisual}
\end{figure}

\begin{figure}
\centering
\subfigure[]
{\label{TinyVisual1}
\includegraphics[width = 0.95\linewidth, clip]{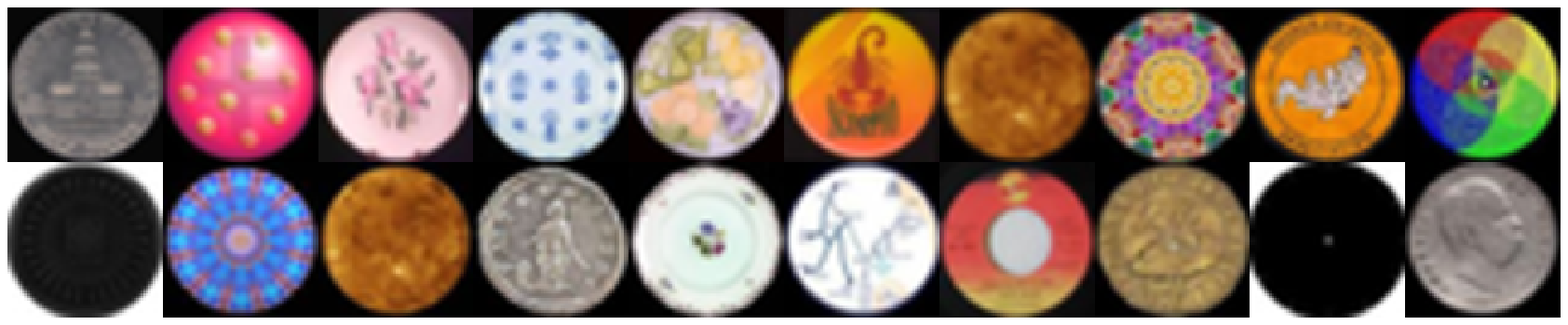}}\\

\subfigure[]
{\label{TinyVisual2}
\includegraphics[width = 0.95\linewidth, clip]{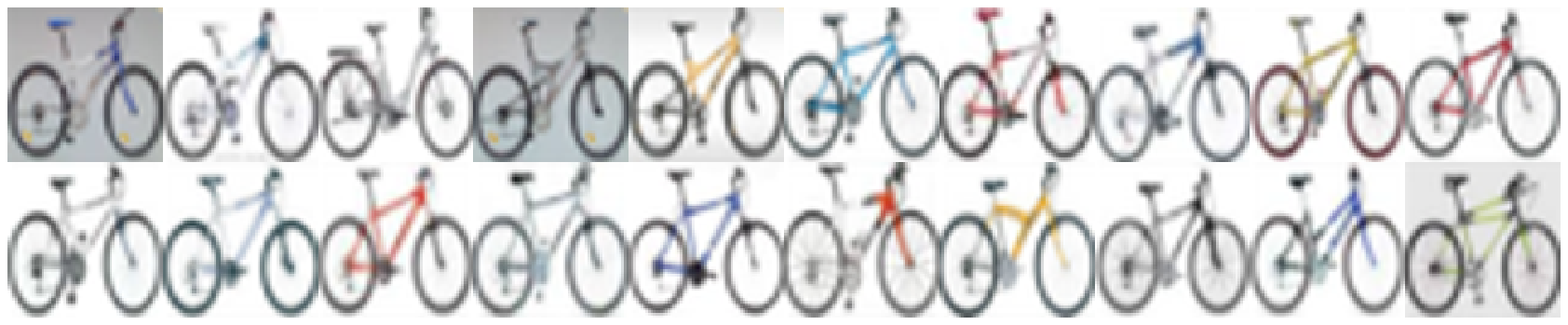}}\\

\subfigure[]
{\label{TinyVisual3}
\includegraphics[width = 0.95\linewidth, clip]{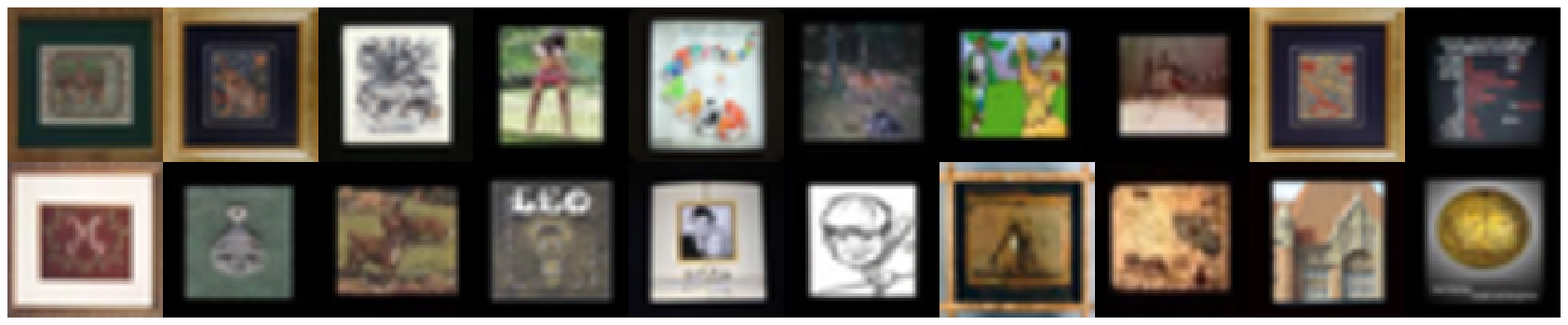}}\\

\subfigure[]
{\label{TinyVisual4}
\includegraphics[width = 0.95\linewidth, clip]{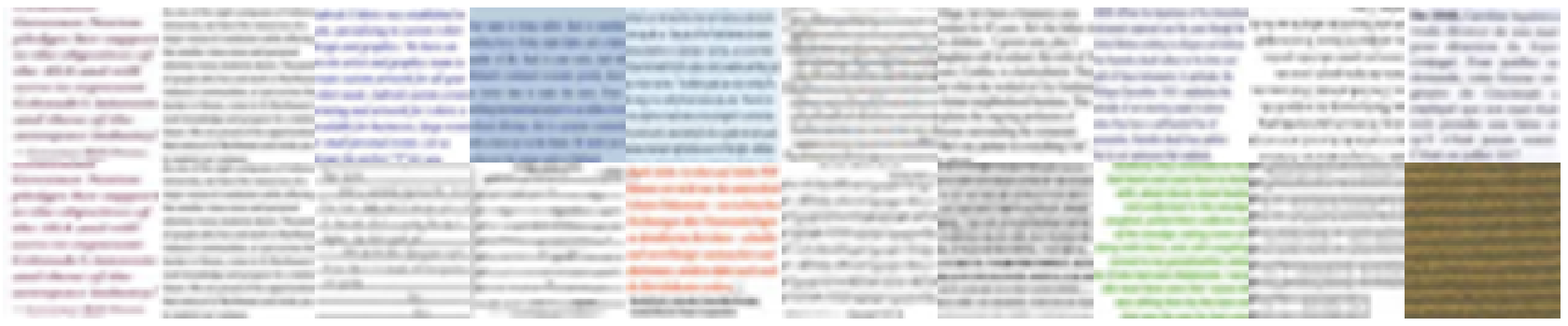}}\\

\subfigure[]
{\label{TinyVisual5}
\includegraphics[width = 0.95\linewidth, clip]{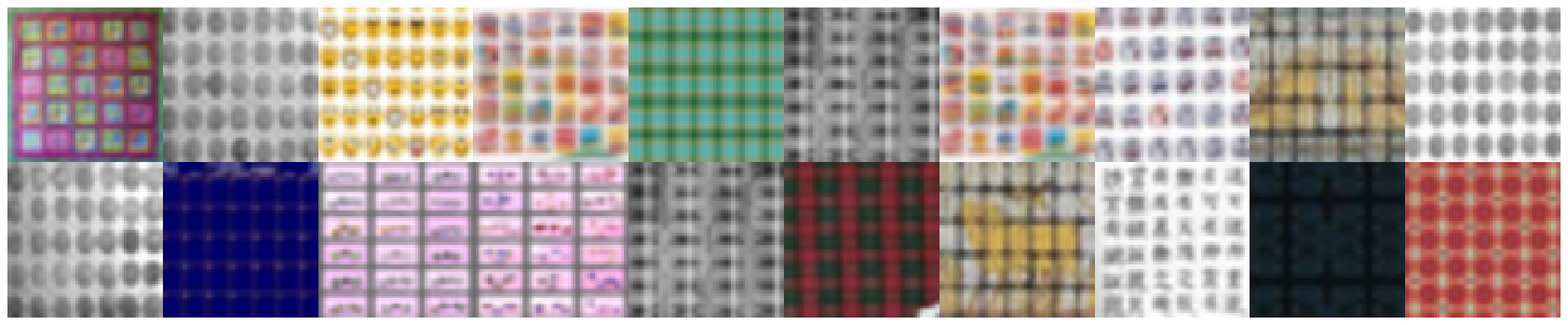}}\\


\subfigure[]
{\label{TinyVisual7}
\includegraphics[width = 0.95\linewidth, clip]{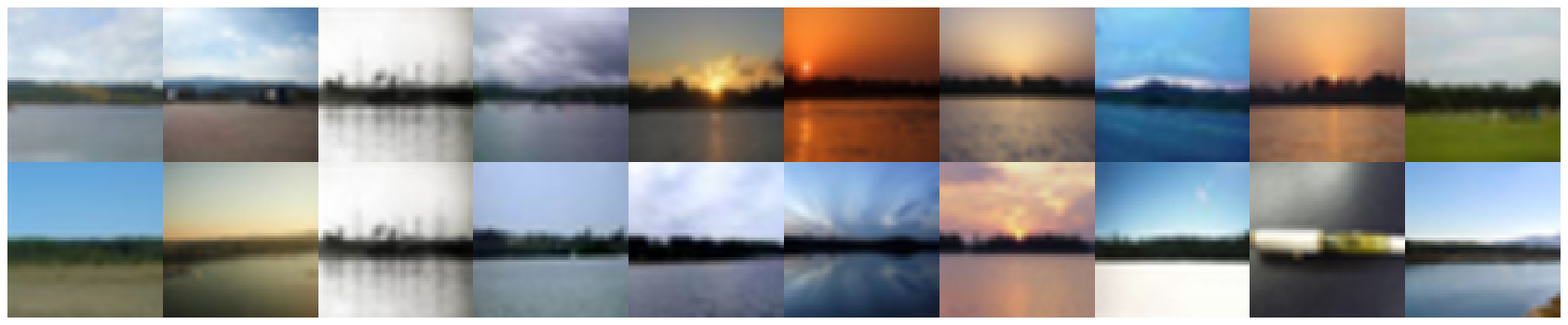}}\\
\caption{Clustering results over the $500K$ tiny images:
each cluster example is represented by two rows of images
which are randomly picked out from the cluster.
}
\label{fig:TinyimageVisual}
\end{figure}

\subsection{Empirical analysis}

We conduct empirical studies to understand why our proposed algorithm has superior performance.
In particular we compare our proposed approach
with AKM~\cite{PhilbinCISZ07} and RAKM~\cite{Philbin10a}
in terms of the accuracy and the time cost of cluster assignment,
using the task of clustering the $1M$ Tiny image data set into $2000$ clusters.
To be on the same ground, in the assignment step the number of candidate clusters
for each point is set the same.
For (R)AKM, the number of candidate clusters is simply the number of points accessed in $k$-d trees
when searching for a nearest neighbor.
For our proposed algorithm, we partition the data points
with RP trees such that the average number of candidate clusters is the same as the number of accessed points in $k$-d trees.
Figure~\ref{fig:studyofourapproachwithAKM:accuray} compares the accuracy of cluster assignment
by varying the number of candidate clusters. We can see that our approach has a much higher accuracy in all cases,
which has a positive impact on the iterative clustering algorithm to make it converge faster.
Figure~\ref{fig:studyofourapproachwithAKM:time} compares the time of performing one iteration, by varying
the number of candidate clusters $M$ for each point.
We can see that our algorithm is much faster than (R)AKM in all cases,
e.g., taking only about half the time of (R)AKM when $M=50$.
This is as expected since finding the best cluster costs $O(1)$
for our algorithm but $O(\log k)$ for $k$d-trees used in (R)AKM.

\begin{figure}
\centering
\subfigure[]{\label{fig:studyofourapproachwithAKM:accuray}\includegraphics[width=0.48\linewidth]{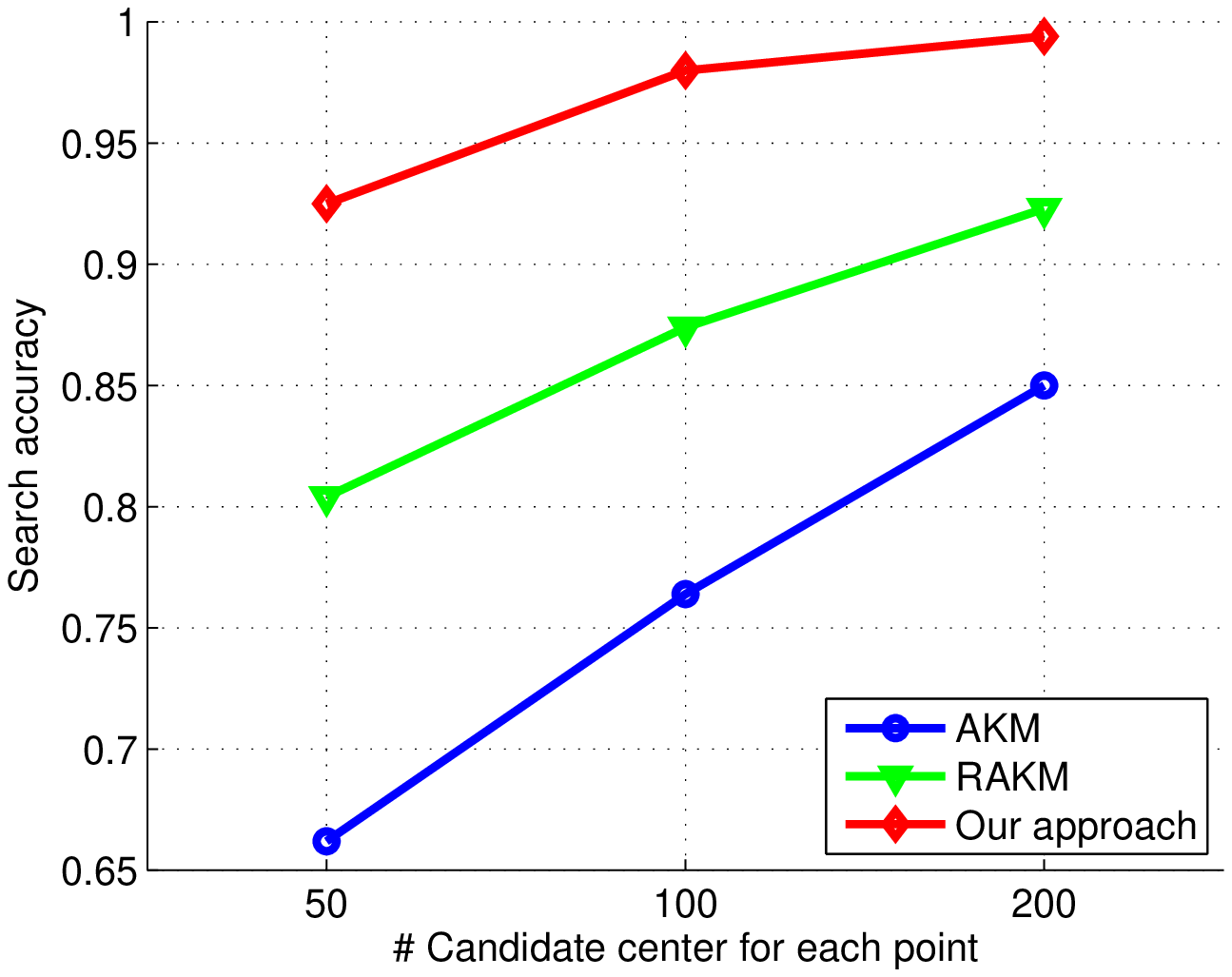}}
~~~\subfigure[]{\label{fig:studyofourapproachwithAKM:time}\includegraphics[width=0.48\linewidth]{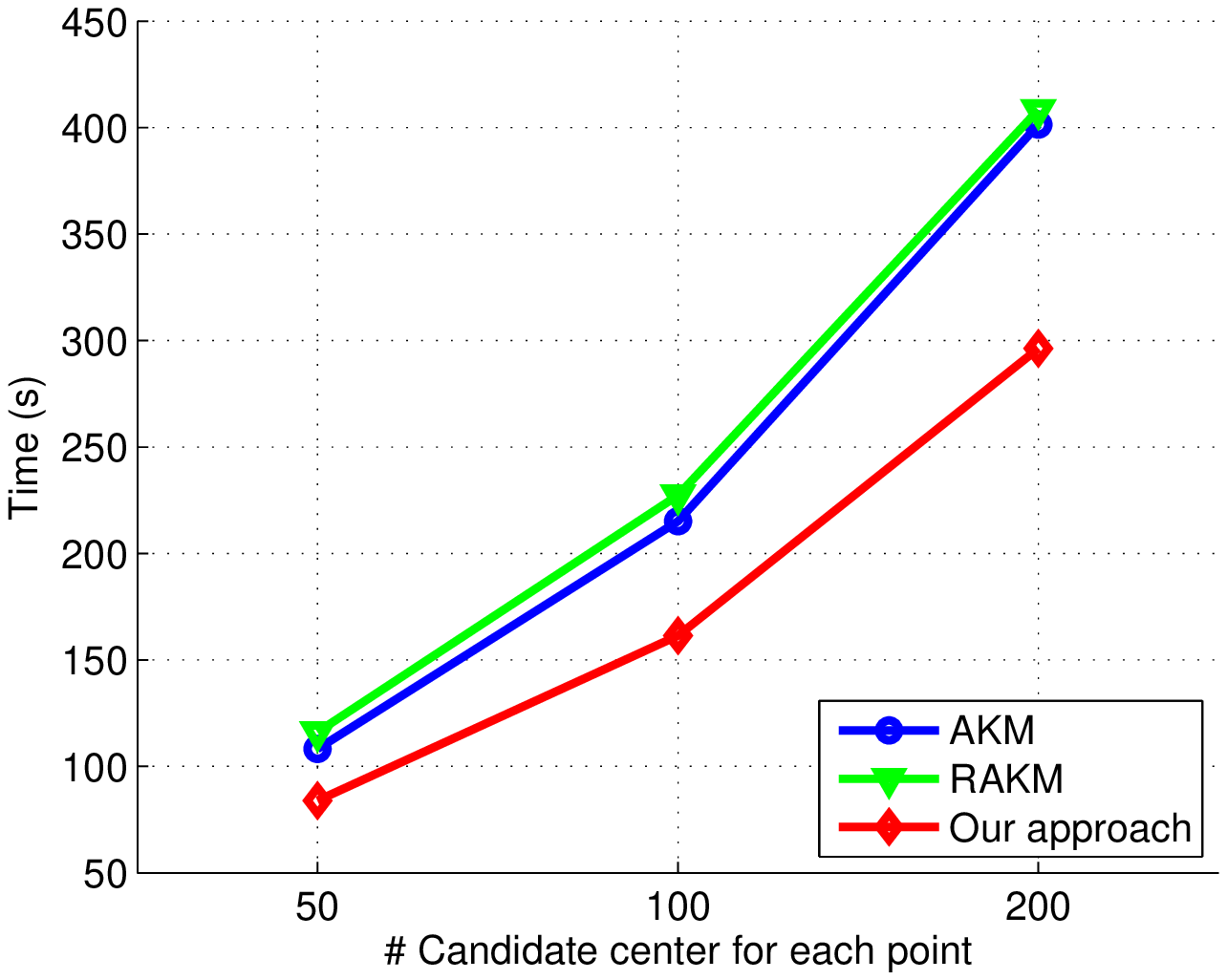}}
\caption{ Comparison of accuracy and time in the assignment step when clustering the $1M$ Tiny image data set into $2000$ clusters.
(a) accuracy vs. the number of cluster candidates; (b) time for one iteration vs. the number of cluster candidates}
\label{fig:studyofourapproachwithAKM}
\end{figure}

We perform another empirical study to investigate the bucket size parameter
in the RP tree, using the task of clustering the SIFT dataset into $10K$ clusters.
Figure~\ref{fig:bucketsize:iteration} shows
the results in terms of WCSSD vs. the number of iterations, with bucket sizes set
to $5$, $10$, $20$, $40$, respectively.
A larger bucket size leads to a larger WCSSD reduction in each iteration,
because it effectively increases the neighborhood size for each data point.
Figure~\ref{fig:bucketsize:time} shows
the result in terms of WCSSD vs. time. We observe that at the beginning,
bucket sizes of $10$ and $20$ perform even better than the bucket size of $40$.
But eventually, the performance of various bucket sizes are similar.
The difference between Figure~\ref{fig:bucketsize:iteration}
and Figure~\ref{fig:bucketsize:time}
is expected, as a larger bucket size leads to a better cluster assignment at each iteration,
but increases the time cost for one iteration.
In our comparison experiments,
a bucket size of $10$ is adopted.

\subsection{Evaluation using object retrieval}

We compare the quality of codebooks built by HKM, (R)AKM, and our approach,
using the performance of object retrieval. AKM and RAKM perform almost
the same when the number of accessed candidate centers is large enough,
so we only present results from AKM.

We perform the experiments on the UKBench $10K$ dataset which has $7M$ local features,
and on the Oxford $5K$ dataset which has $16M$ local features.
Following~\cite{PhilbinCISZ07},
we perform the clustering algorithms to build the codebooks, and test only the filtering stage of the retrieval system,
i.e., retrieval is performed using the inverted file (including the tf-idf weighting).

The results over the UKbench $10K$ dataset are obtained
by constructing $1M$ codebook, and use the $L_1$ distance metric.
The results of HKM and AKM
are taken from~\cite{NisterS06}
and~\cite{PhilbinCISZ07}, respectively.
From Table~\ref{tab:ukbench:tab},
we see that for the same codebook size,
our method outperforms other approaches.
Besides, we also conduct the experiment
over subsets of various sizes, which means that
we only consider the images in the subset as queries
and the search range is also constrained within the subset.
The performance comparison is given
in Figure~\ref{fig:ukbench:fig},
from which we can see our approach consistently
gets superior performances.

\begin{figure}[t]
\centering
\subfigure[]{\label{fig:bucketsize:iteration}\includegraphics[width = .45\linewidth]{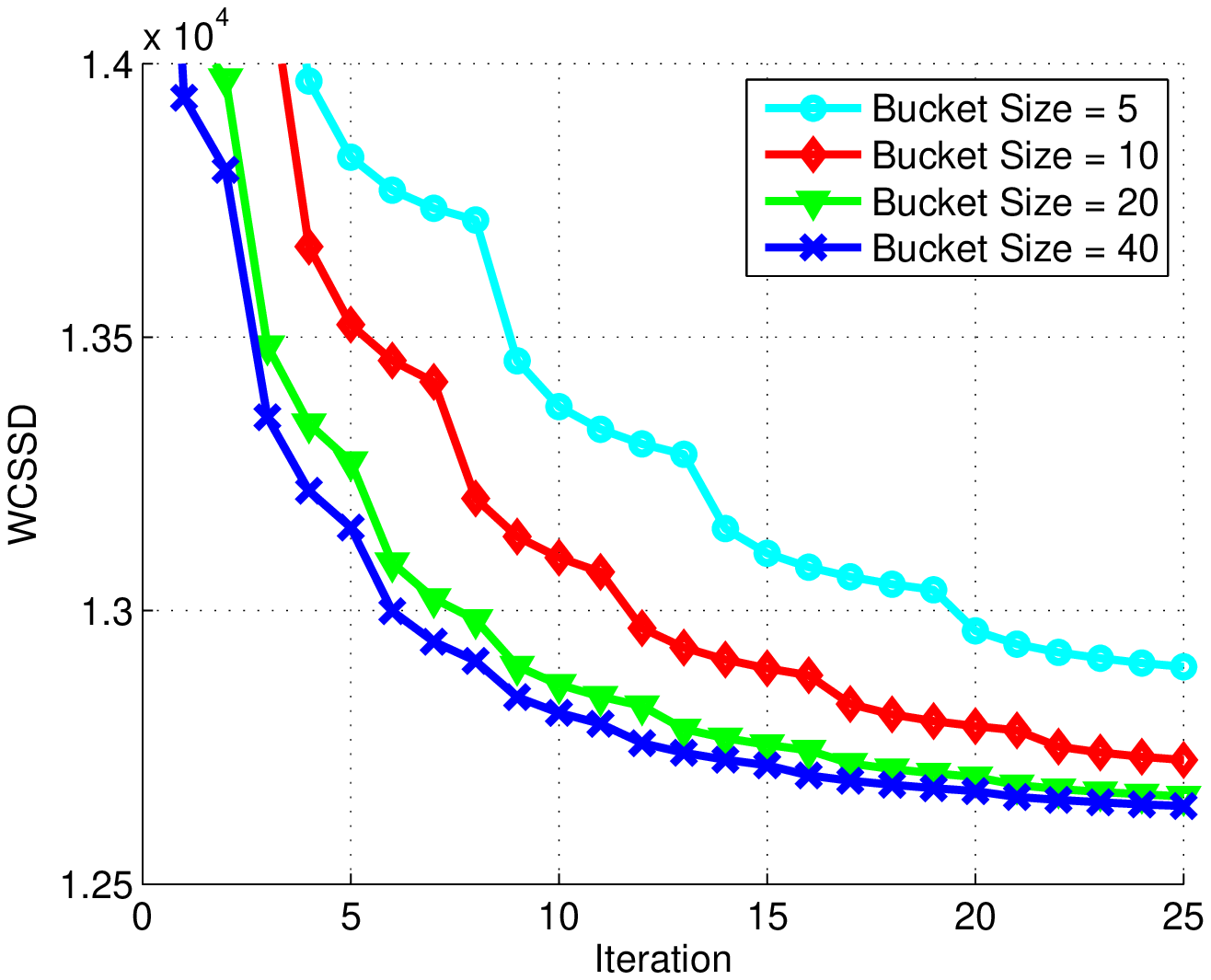}}
~~~~~\subfigure[]{\label{fig:bucketsize:time}\includegraphics[width = .45\linewidth]{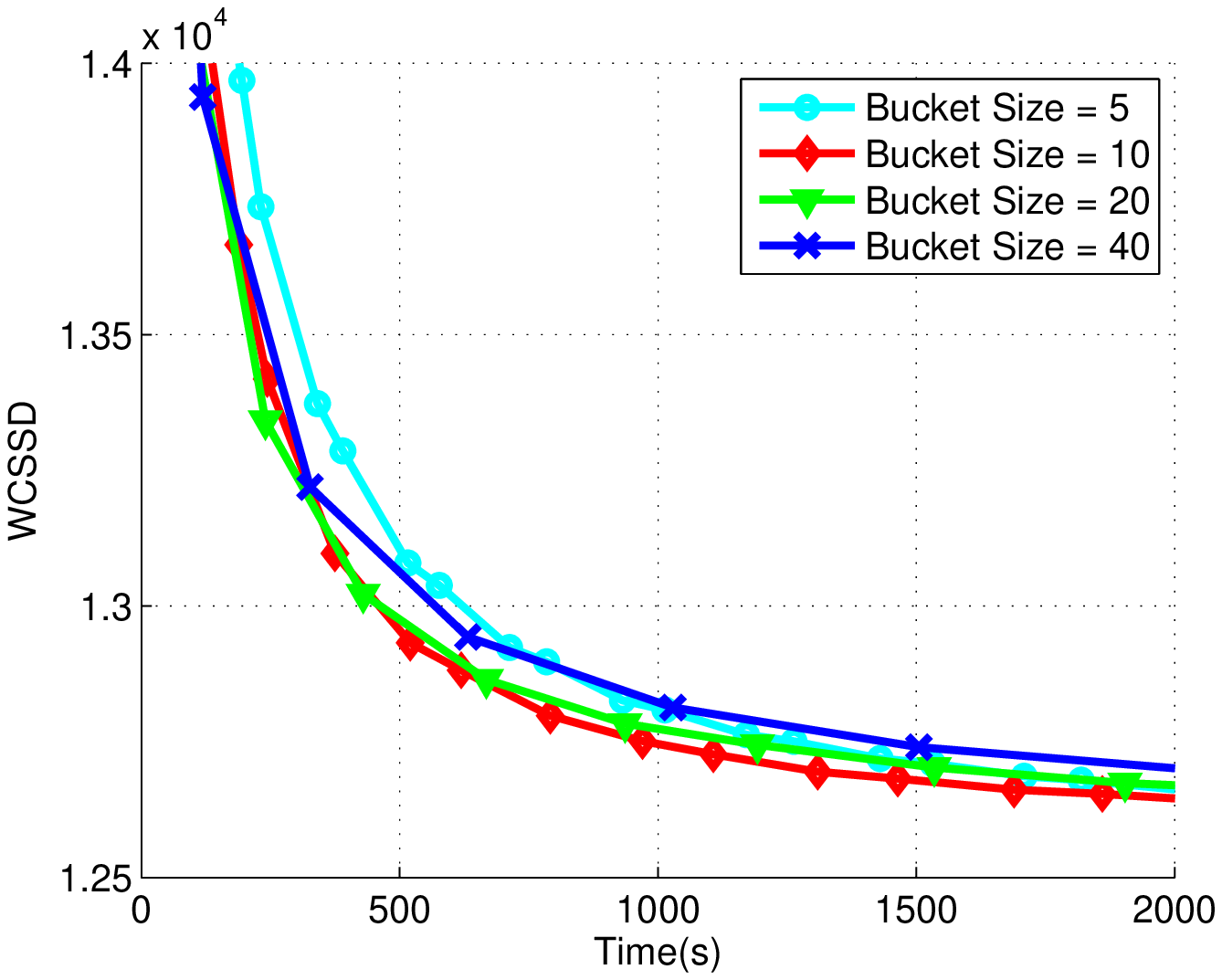}}
\caption{Clustering performance vs. the bucket size of a RP tree
}
\label{fig:bucketsize}
\end{figure}

\begin{table}
\sidecaption[t]
\begin{tabular}[b]{|@{~~}c@{~~}|@{~~}c@{~~}|@{~~}c@{~~}|}
\hline
Method & Scoring levels & Average Top \\
\hline
HKM & $1$ & $3.16$\\
HKM & $2$ & $3.07$\\
HKM & $3$ & $3.29$\\
HKM & $4$ & $3.29$\\
AKM & & $3.45$\\
Ours & & $\mathbf{3.50}$\\ \hline
\end{tabular}
\caption{A comparison of our approach
to HKM and AKM
on the UKbench $10K$ data set
using a $1M$-word codebook}
\label{tab:ukbench:tab}
\end{table}

\begin{figure}[t]
\sidecaption[t]
\includegraphics[width = .45\linewidth, clip]{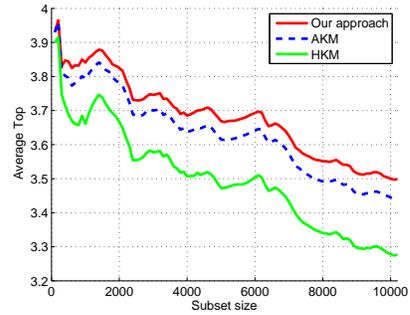}
\caption{A comparison of our approach
to HKM and AKM
on the UKbench $10K$ data set
with various subset sizes}
\label{fig:ukbench:fig}
\end{figure}

The performance comparison using the Oxford $5K$ dataset
is shown in Table~\ref{tab:oxford5k1}.
We show the results of using the bag-of-words (BoW) representation with a $1M$ codebook
and using spatial re-ranking~\cite{PhilbinCISZ07}.
Our approach achieves the best performance,
outperforming AKM in both the BoW representation and spatial re-ranking.
We also compare the performance
of our approach to AKM and HKM
using different codebook sizes, as shown
in Table~\ref{tab:oxford5k2}.
Our approach is superior compared to other approaches
with different codebook sizes.
Different from AKM that gets the best performance
with a $1M$-word codebook,
our approach obtains the best performance with a $750K$-word codebook,
indicating that our approach is producing a higher quality codebook.

Last,
we show
some visual examples
of the retrieval results
in Figure~\ref{fig:OxfordVisual}.
The first images in each row is the query,
followed by the top results.

\begin{figure*}
\begin{minipage}{1\linewidth}
\centering
\subfigure[]
{\label{oxfordvisual1}
\includegraphics[width = 0.95\linewidth, clip]{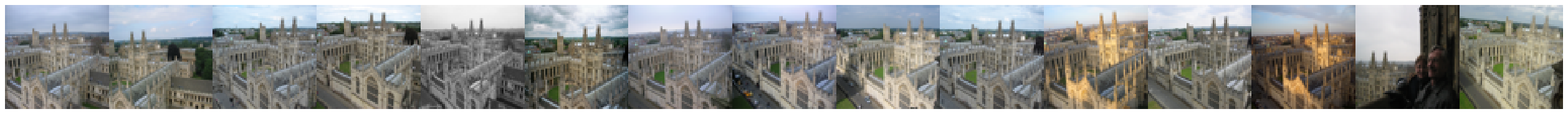}}\\

\subfigure[]
{\label{oxfordvisual2}
\includegraphics[width = 0.95\linewidth, clip]{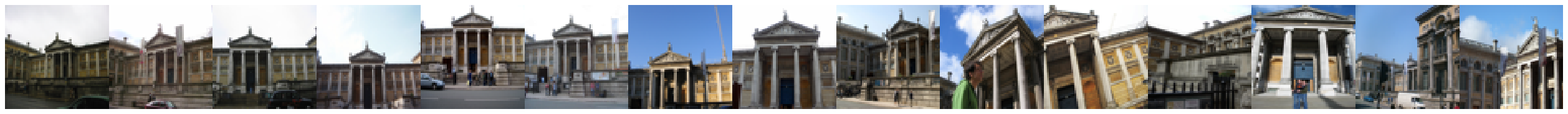}}\\

\subfigure[]
{\label{oxfordvisual3}
\includegraphics[width = 0.95\linewidth, clip]{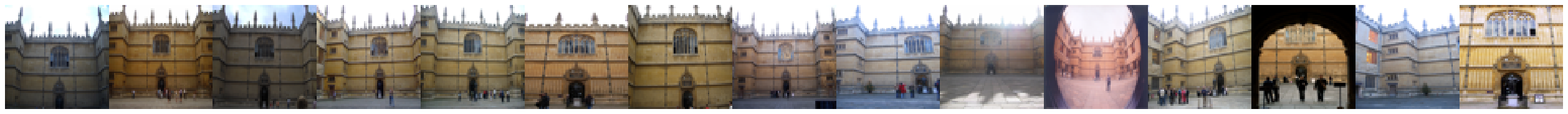}}\\

\subfigure[]
{\label{oxfordvisual4}
\includegraphics[width = 0.95\linewidth, clip]{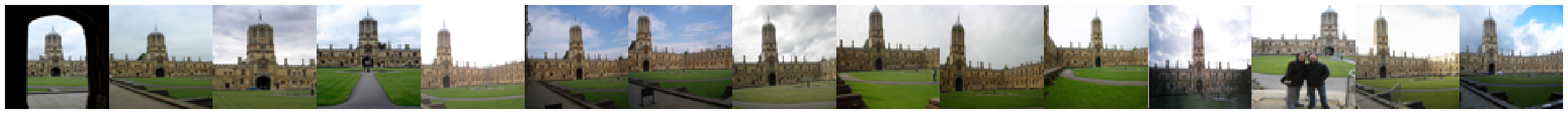}}\\

\subfigure[]
{\label{oxfordvisual5}
\includegraphics[width = 0.95\linewidth, clip]{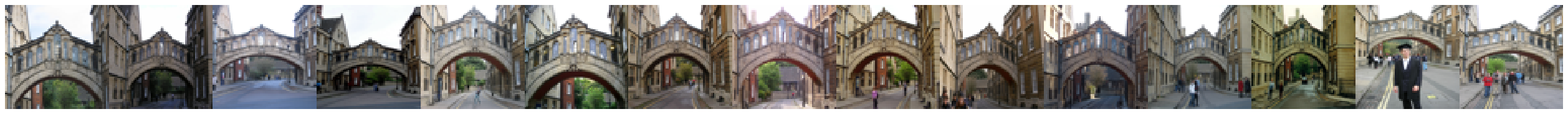}}\\

\subfigure[]
{\label{oxfordvisual6}
\includegraphics[width = 0.95\linewidth, clip]{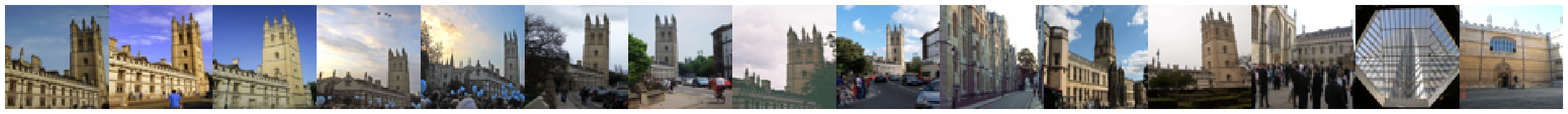}}\\

\subfigure[]
{\label{oxfordvisual7}
\includegraphics[width = 0.95\linewidth, clip]{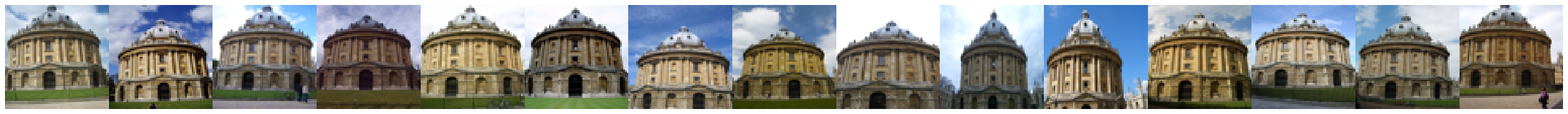}}\\
\caption{Examples of the retrieval results
of Oxford5k dataset: the first image in each row
is the query image
and the following images are the top results
}
\label{fig:OxfordVisual}
\end{minipage}
\end{figure*}

\section{Conclusions}
There are three factors
that contribute to the superior performance of our proposed approach:
(1) We only need to consider active points that change their cluster assignments
in the assignment step of the $k$-means algorithm;
(2) Most active points locate at or near cluster boundaries;
(3) We can efficiently identify active points
by pre-assembling data points using multiple random partition trees.
The result is a simple, easily parallelizable,
and surprisingly efficient $k$-means clustering algorithm.
It outperforms state-of-the-art on clustering large-scale real datasets and learning codebooks for image retrieval.

\begin{table}[b]
\centering
\caption{A comparison of our approach
with HKM and AKM
on the Oxford $5K$ data set
with a $1M$-word codebook}
\label{tab:oxford5k1}
\begin{tabular}[b]{|@{~~}c@{~~}|@{~~}c@{~~}|@{~~}c@{~~}|@{~~}c@{~~}|}
\hline
Method & Scoring level& mAP (BoW) & mAP (Spatial)\\
\hline
HKM-1 & $1$ & $0.439$ & $0.469$\\
HKM-2 & $2$ & $0.418$ & \\
HKM-3 & $3$ & $0.372$ & \\
HKM-4 & $4$ & $0.353$ & \\
AKM & & $0.618$ & 0.647\\
Our approach & & $\mathbf{0.655}$ & $\mathbf{0.666}$\\
\hline
\end{tabular}
\end{table}

\begin{table}[t]
\centering
\caption{Performance comparison of our approach,
HKM, and AKM
using different codebook sizes
on the Oxford $5K$ data set}
\label{tab:oxford5k2}
\begin{tabular}{|@{~~}c@{~~}|@{~~}c@{~~}|@{~~}c@{~~}|@{~~}c@{~~}|@{~~}c@{~~}|@{~~}c@{~~}|}
\hline
Vocabulary size & HKM & AKM & AKM spatial & Ours & Ours spatial \\
\hline
$250K$ & $0.399$ & $0.598$ & $0.633$ & $0.620$ & $0.636$\\
$500K$ & $0.422$ & $0.606$ & $0.642$ & $0.647$ & $0.658$\\
$750K$ & $0.440$ & $0.609$ & $0.630$ & $\mathbf{0.664}$ & $\mathbf{0.674}$\\
$1M$ & $0.439$ & $0.618$ & $0.645$ & $0.655$ & $0.666$\\
$1.25M$ & $0.449$ & $0.602$ & $0.625$ & $0.650$ & $0.674$\\
$2M$ & $0.457$ & $0.604$ & $0.617$ & $0.621$ & $0.647$\\
\hline
\end{tabular}
\end{table}

{
\small
\bibliographystyle{./styles/spmpsci}
\bibliography{clustering}
}
\end{document}